\documentclass{article}
\usepackage{iclr2021_conference,times}
\iclrfinalcopy
\usepackage[utf8]{inputenc}
\usepackage{amsmath}
\usepackage{amsthm}
\usepackage{graphicx}
\usepackage{amssymb}
\usepackage{todonotes}
\usepackage{hyperref}
\usepackage{verbatim}
\usepackage{algorithm}
\usepackage[noend]{algpseudocode}
\usepackage{bbm}
\usepackage{mathtools}
\usepackage{float}
\usepackage{thm-restate}
\usepackage{caption}
\usepackage{subcaption}

\usepackage[utf8]{inputenc} % allow utf-8 input
\usepackage[T1]{fontenc}    % use 8-bit T1 fonts
\usepackage{hyperref}       % hyperlinks
\usepackage{url}            % simple URL typesetting
\usepackage{booktabs}       % professional-quality tables
\usepackage{amsfonts}       % blackboard math symbols
\usepackage{nicefrac}       % compact symbols for 1/2, etc.
\usepackage{microtype}      % microtypography
\usepackage{paralist}

\newtheorem{theorem}{Theorem}
\newtheorem{corollary}{Corollary}

\newtheorem{definition}{Definition}
\newtheorem{conjecture}{Conjecture}
\newtheorem{example}{Example}
\newtheorem{lemma}{Lemma}
\newtheorem{claim}{Claim}
\newtheorem{remark}{Remark}
\newtheorem{fact}{Fact}

\algnewcommand\CommentLine[1]{
     \hspace{2pt}$\triangleright$ \hspace{2pt}\text{#1}
     }

\newcommand{\eps}{\epsilon}

\newcommand{\norm}[1]{\| #1 \|}
\newcommand{\ip}[2]{\langle #1 , #2 \rangle}
\newcommand{\VS}[1]{{\color{red} \textbf{VS}: #1}}

\newcommand{\defeq}{\overset{\mathrm{def}}=}

\newcommand{\code}{\mathbf{c}}

\usepackage{ifthen}
\newboolean{long}   
\setboolean{long}{false}   
\ifthenelse{\boolean{long}}
{

}
{

}

\usepackage{xcolor}

\newif\ifcomments
\commentstrue % uncomment to turn all colored comments on
\ifcomments
\newcommand{\aga}[1]{{\color{blue}#1}} % comment color for Atish
\else
\newcommand{\aga}[1]{}
\fi

\DeclarePairedDelimiterX{\infdivx}[2]{(}{)}{%
  #1\;\delimsize\|\;#2%
}

\newcounter{this-list}
\newenvironment{tightenumerate}{
\vspace{2pt}
\begin{list}{\arabic{this-list}.}{\usecounter{this-list}
                                 \setcounter{this-list}{0}
  \setlength{\itemsep}{0pt}%
  \setlength{\parsep}{0pt}%
  \setlength{\topsep}{0pt}%
    \setlength{\partopsep}{0pt}%
  \setlength{\leftmargin}{3.5ex}%
  \setlength{\labelwidth}{4.5ex}%
  \setlength{\labelsep}{1ex}%
}} {\end{list}\vspace{2pt}}

\usepackage{booktabs}       % professional-quality tables

\newcommand{\erf}{\text{erf}}
\newcommand{\expect}{{\rm E}}
\newcommand{\m}[1]{\mathbf{#1}} % Boldface for matrices
\newcommand{\sm}[1]{\boldsymbol{#1}}
\newcommand{\tpose}{{\rm T}}

\newcommand{\xb}{\mathbf{x}}
\newcommand{\cb}{\mathbf{c}}
\newcommand{\poly}{\text{poly}}

\newcommand{\app}{\text{app}}
\newcommand{\bbet}{\sm{\beta}}
\newcommand{\balpha}{\sm{\alpha}}
\newcommand{\bnorm}{\beta}
\newcommand{\x}{\m{x}}

\newcommand{\y}{\m{y}}
\newcommand{\F}{\m{F}}
\renewcommand{\H}{\m{H}^{\infty}}

\newcommand{\X}{\m{X}}
\newcommand{\ab}{\mathbf{a}}
\newcommand{\yb}{\mathbf{y}}
\renewcommand{\a}{\alpha}

\newcommand{\rrat}{R}
\newcommand{\h}{\m{h}}
\newcommand{\w}{\m{w}}

\newcommand{\Lo}{\mathcal{L}}

\title{One Network Fits All? Modular versus Monolithic Task Formulations in Neural Networks}
\vspace{5pt}
\author{
Atish Agarwala \& Abhimanyu Das\\
Google Research\\
\{thetish,abhidas\}@google.com

\And Brendan Juba\\
Washington U.\ St.\ Louis\thanks{Work performed in part while visiting Google.}\\
bjuba@wustl.edu

\And Rina Panigrahy\\
Google Research\\
rinap@google.com

\And Vatsal Sharan\\
MIT\thanks{Work performed in part while affiliated with Stanford, and in part while interning at Google.}\\
vsharan@mit.edu

\And Xin Wang \& Qiuyi Zhang\\
Google Research\\
\{wanxin,qiuyiz\}@google.com
}

\begin{document}

\maketitle
\vspace{-10pt}
\begin{abstract}

 Can deep learning solve multiple tasks simultaneously, even when they are unrelated and very different? We investigate how the representations of the underlying tasks affect the ability of a single neural network to learn them jointly. We present theoretical and empirical findings that a single neural network is capable of simultaneously learning multiple tasks from a combined data set, for a variety of methods for representing tasks---for example, when the distinct tasks are encoded by well-separated clusters or decision trees over certain task-code attributes. More concretely, we present a novel analysis that shows that families of simple programming-like constructs for the codes encoding the tasks are learnable by two-layer neural networks with standard training. We study more generally how the complexity of learning such combined tasks grows with the complexity of the task codes; we find that combining many tasks may incur a sample complexity penalty, even though the individual tasks are easy to learn. We provide empirical support for the usefulness of the learning bounds by training networks on clusters, decision trees, and SQL-style aggregation.

\end{abstract}
%\vspace{-8pt}
\section{Introduction}
Standard practice in machine learning has long been to only address carefully circumscribed, often very related tasks. For example, we might train a single
classifier to label an image as containing objects from a certain predefined
set, or to label the words of a sentence with their semantic roles. % of those words in that sentence. 
Indeed, when working with relatively simple classes of functions
like linear classifiers, it would be unreasonable to expect to train a
classifier that handles more than such a carefully scoped task (or related tasks in standard multitask learning). As techniques
for learning with relatively rich classes such as neural networks have been
developed, it is natural to ask whether or not such scoping of tasks is
inherently necessary. Indeed, many recent works (see Section~\ref{related-work-sec}) have proposed eschewing this careful scoping of tasks, and instead training a single, ``monolithic''
function spanning many tasks.

Large, deep neural networks can, in
principle, represent multiple classifiers in such a monolithic
learned function \citep{hornik1991approximation}, giving rise to the field of multitask learning.
%\VS{So in the abstract "representation" was the input representation, here it is the learned representation, we are sort of going back and forth later as well, could be confusing}.
This combined function might be 
learned by combining all of the training data for all of the tasks into one 
large batch--see Section~\ref{related-work-sec} for some examples. Taken to an extreme, we could consider seeking to learn a 
{\em universal} circuit---that is, a circuit that interprets arbitrary programs 
in a programming language which can encode various tasks.
But, the ability to \emph{represent} such a monolithic combined function does not 
necessarily entail that such a function can be  efficiently \emph{learned} by existing 
methods. Cryptographic hardness theorems \citep{kearns1994cryptographic}
establish that this is not possible in general by {\em any} method, let alone the
specific training methods used in practice. Nevertheless, we still can ask
how rich a family of tasks can be learned by these standard methods. 
In this work, we study the extent to which backpropagation with
stochastic gradient descent (SGD) can learn such monolithic 
functions on diverse, unrelated tasks. There might still be some inherent benefit to an architecture in which tasks are partitioned into sub-tasks of such small scope, and the training data is correspondingly partitioned prior to learning. For example, in the early work on multitask learning, \citet{caruana1997multitask} observed that training a network to solve unrelated tasks simultaneously seemed to harm the overall performance. Similarly, the seminal work of \citet{jacobs1991adaptive} begins by stating that \emph{``If backpropagation is used to train a single, multilayer network to perform different subtasks on different occasions, there will generally be strong interference effects that lead to slow learning and poor generalization''}.  We therefore ask if, for an unfortunate choice of tasks in our model, learning by standard methods might be fundamentally impaired.

As a point of reference from neuroscience, the classical view is that distinct tasks are handled in
the brain by distinct patches of
the cortex. While it is a subject of debate whether modularity 
exists for higher level tasks
\citep{samuels2006mind}, it is accepted that there are dedicated 
modules for low-level tasks such as vision and audio processing.
Thus, it seems that
the brain produces a {\em modular} architecture, in which different tasks are
handled by different regions of the cortex. 
Conceivably, this
division into task-specific regions might be driven by fundamental
considerations of learnability: A single, monolithic neural circuit
might simply be too difficult to learn because the different tasks might interfere with one another. Others have taken neural    
networks trained by backpropagation as a model of learning in the 
cortex \citep{musslick2017multitasking}; to the extent that this is reasonable, our work has some bearing on these questions as well. 

\iffalse
OLD VERSION----
As a point of reference, we recall that a classical view of the brain, going
back to Brodmann, asserts that distinct tasks are handled by distinct patches of
the cortex. While it is the subject of significant debate whether modularity 
exists for higher level reasoning and cognitive tasks
\citep{samuels2006mind}, it is more or less accepted that there are dedicated 
modules for low-level tasks such as vision and audio processing.
Thus, in contrast to these monolithic representations, it seems that
the brain produces {\em modular} representations, in which different tasks are
handled by different regions of the cortex. 
One natural hypothesis is that such
localization is simply driven by anatomy---i.e., by the proximity of cells in
different regions of the brain to the nerve endings carrying different sensory
input or motor commands. But an alternative hypothesis might be that this
division into somewhat task-specific regions is driven by fundamental
considerations of learnability. Conceivably, a single, monolithic neural circuit
could simply be too difficult to learn because the learning process for
different tasks might interfere with one another. To the extent that neural    
networks trained by backpropagation are a reasonable model of learning in the 
cortex, our work has some bearing on these issues as well. 
END OLD VERSION
\fi

\subsection{Our results}

\begin{figure}
    \centering
    \includegraphics[width=\linewidth]{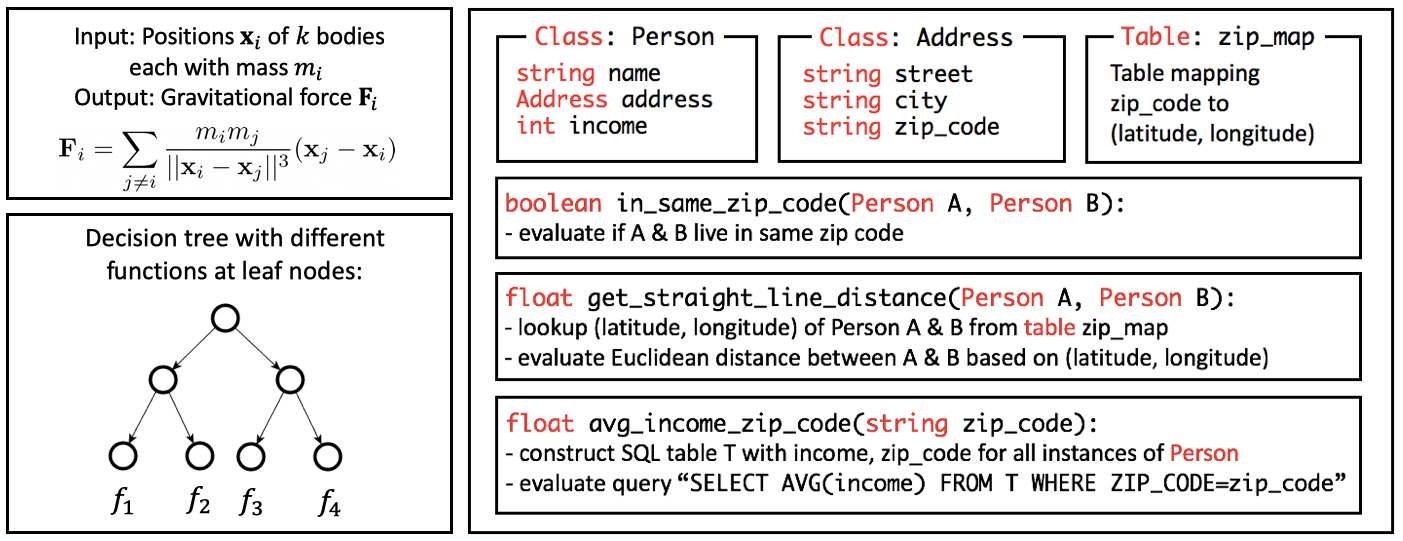}
    \caption{Our framework shows that it is possible to learn analytic functions such as the gravitational force law, decision trees with different functions at the leaf nodes, and programming constructs such as those on the right, all using a non-modular monolithic architecture.}
    %\aga{I would reverse the ordering - analytic, then decision, then programs. Also, currently reads a bit funny-can edit once we decide ordering.} \textcolor{red}{Abhi: This figure doesnt seem to be referenced in the text} }
    \label{fig:learn}
\end{figure}

We find, perhaps surprisingly, that combining multiple tasks into one cannot fundamentally impair learning with standard training methods. We demonstrate this for a broad family of methods
for combining individual tasks into a single monolithic task. For example, inputs for each individual tasks may come from a disjoint region (for example, a disjoint ball) in a common  input space, and each individual task could then involve applying some arbitrary simple function (e.g., a separate linear classifier for each region).
Alternately there may be an explicit {\em ``task code''} attribute (e.g., a one-hot code), together with the usual input attributes and output label(s), where 
examples with the same task code are examples for the same learning task. Complementing our results that combining multiple tasks does not impair learning, we also find that some task coding schemes do incur a sample complexity penalty. 
A vast variety of task coding schemes may be used. As a concrete example, when the data points for each task are well-separated into
distinct clusters, and the tasks are linear classification tasks, we show that a two-layer architecture trained with SGD %\textcolor{red}{Abhi: maybe "there exists a two-layer architecture" instead of "such a"?}
successfully learns the
%\aga{a representation? Or is the most efficient one?}
combined, monolithic function; the required amount of data 
simply scales as the sum of the amount required to learn each task individually (Theorem \ref{thm:informa_clusterfunctions}).
Meanwhile, if the tasks are determined by a balanced decision tree of height $h$
on $d$ code attributes (as in Fig.~\ref{fig:learn}, left), we find that the training time and amount of data needed
scales as $\sim d^{h}$---quasipolynomial in the $2^h$ leaves (distinct tasks)
when $d$ is of similar size to $h$, and thus when the coding is efficient (Theorem \ref{thm:decisiontree_informal}). We also prove  a corresponding lower bound, which shows that this bound
is in fact asymptotically tight (Theorem \ref{thm:decisiontree_informal}). More generally, for task 
codings based on decision trees using linear splits with a margin of at least 
$\gamma$ (when the data has unit $\ell_2$ norm), the  training time
%\aga{for training?} 
and required data are 
asymptotically bounded by $\sim e^{O(h/\gamma^2)}$, which for constant $\gamma$ is polynomial in the $2^h$ functions (Theorem \ref{thm:decisiontreemargin_informal}).
%\aga{is that bad or good?}. 

We generalize from these cluster-based and decision-tree based task codings to more complex codes that are actually simple programs. For instance, we show
that SQL-style aggregation queries over a fixed database, written as a functions of 
the parameters of the query, can also be learned this way. More generally, simple programming constructs (such as in Fig.~\ref{fig:learn}, right), built by operations such as compositions, aggregation, concatenation, and branching  on a small
number of such
learnable functions, are also learnable (Theorem \ref{thm:program_informal}). In general, we can learn a low-depth formula 
(circuit with fan-out 1) in which each gate is not merely a switch (as in a decision tree), but can be any analytic function
on the inputs, 
including arithmetic operations.
%\textcolor{brown}{Rich: I think this paragraph should be above? In the first paragraph?}
Again, our key technical contribution is that we show that all of these functions are efficiently learned by SGD.
This is non-trival since, although universal approximation
theorems show that such functions can be expressed by (sufficiently wide) two-layer neural networks, under standard assumptions
some expressible functions are 
%\aga{believed or proven?} Proven assuming the hardness of SVP. The hardness of SVP is a belief, but reasonably strongly held. I tried to revise to clarify --BJ
not learnable
\cite{klivans2009cryptographic}. We supplement the theoretical bounds with experiments on clusters, decision trees, and SQL-style aggregation showing that
such functions are indeed learned in practice.

We note that the learning of such combined functions could have been engineered by hand:
%\aga{representations+decoders, right?}: 
for example, there exist 
efficient algorithms for learning clusterings or such decision trees, and it
is easy to learn the linear classifiers given the partitioned 
data. Likewise, these classes of functions are all known to be learnable by other 
methods, given an appropriate transformation of the input features. The key point is 
that {\em the two-layer neural network can jointly learn
the task coding scheme and the task-specific functions without special 
engineering of the architecture}. That is, it is unnecessary to engineer a way 
of partitioning of the data into separate tasks prior to learning.
Relatedly,  the time and sample requirements of learning multiple tasks on a single network in general is insufficient to 
explain the modularity observed in biological neural networks if their learning dynamics are similar to
SGD ---i.e., we cannot explain the presence of modularity from such 
general considerations.

All our theoretical results are based upon a fundamental theorem that shows that analytic functions can be efficiently 
learnt 
by wide (but finite-width) two-layer neural networks with standard activation functions (such
as ReLU), using SGD from a random initialization. Specifically, we derive novel generalization 
bounds 
for multivariate analytic functions (Theorems~\ref{thm:analytic-informal} and~\ref{thm:multivar}) by relating wide 
networks to
kernel learning with a specific network-induced kernel 
\citep{jacot_neural_2018,du_gradient_2019,allen-zhu_learning_2019,arora_finegrained_2019,lee_wide_2019}, known as 
the  \emph{neural tangent kernel} (NTK) \citep{jacot_neural_2018}. We further develop a \emph{calculus of bounds}
showing that the sum, product, ratio, and composition of analytic functions is also learnable, with bounds constructed 
using 
the familiar product and chain rules of
univariate calculus (Corollaries \ref{cor:prod_rule}, 
\ref{cor:chain_rule}).  These above learnability results may be of independent interest; for 
example, they can be used to show that natural physical laws like the gravitational force equations (shown in 
Fig.~\ref{fig:learn}) can be efficiently learnt 
by neural networks (Section~\ref{sec:learning_gravity}). Furthermore, our bounds imply that the NTK kernel for ReLU 
activation has theoretical learning guarantees that are superior to the Gaussian kernel (Section~\ref{sec:analytic}), 
which we also demonstrate empirically  with experiments on learning the gravitational force law 
(Section~\ref{sec:gravity_experiments}). 
\vspace{-7pt}

\subsection{Related work}\label{related-work-sec}
%\paragraph{Empirical demonstrations of many tasks on one network.}
Most related to our work are a number of works in application areas that have 
sought to learn a single network that can perform many different tasks.
In natural language processing, \citet{tsai2019small} show that a single model 
can solve machine translation across more than 50 languages. Many other works in NLP similarly seek to use one model for multiple languages, or even multiple 
tasks \citep{johnson2017google,aharoni2019massively,bapna2019massively, devlin2018bert}. Monolithic models have also been successfully trained for tasks in very different domains, such as speech and language \citep{kaiser2017one}. Finally,
there is also work on training extremely large neural networks which have the capacity to learn multiple tasks \citep{shazeer2017outrageously,raffel_exploring_2019}.
These works provide empirical clues that suggest that a single network can successfully be trained to perform a wide variety of tasks. But, they do not provide a systematic theoretical investigation of the extent of this ability as we do here.

%\paragraph{Multitask learning.}
\citet{caruana1997multitask} proposed {\em multitask learning} in which a single
network is trained to solve multiple tasks on the same input simultaneously, as
a vector of outputs. He observed that average generalization error for the multiple tasks may be much better than when the tasks are trained separately,
and this observation initiated an active area of machine learning research~\citep{zhang2017survey}. Multitask learning is obviously related to our monolithic
architectures. The difference is that whereas in multitask learning all of the 
tasks are computed simultaneously and output on separate gates, here all of the
tasks share a common set of outputs, and the task code inputs switch between the various tasks. Furthermore, contrary to the main focus of multitask learning, we are primarily interested in the extent to which different tasks may interfere, rather than how much similar ones may benefit.

%\paragraph{Cognitive models of multitasking.}
Our work is also related to studies of neural models of
multitasking in cognitive science. In particular, 
\citet{musslick2017multitasking} consider a similar two-layer architecture in
which there is a set of task code attributes. But, as in multitask 
learning, they are interested in how many of these tasks can be performed 
simultaneously, on distinct outputs. They analyze the tradeoff between improved sample complexity and interference of the tasks with a handcrafted ``gating'' 
scheme, in which the parts of activity are zeroed out depending on the input (as opposed to the usual nonlinearities); in this model, they find out that the
speedup from multitask learning comes at the penalty of limiting the number of
tasks that can be correctly computed as the similarity of inputs varies. Thus,
in contrast to our model where the single model is computing distinct tasks
sequentially, they do find that the distinct tasks can interfere with each other when we seek to solve them simultaneously.

\section{Technical Overview}
\label{sec:techinical_overview}
We now give a more detailed overview of our theoretical techniques and results, with informal statements of our main theorems. For full formal statements and proofs, please see the Appendix.

\subsection{Learning Analytic Functions} \label{sec:learning_analytic}

Our technical starting point is to generalize the analysis of ~\citet{arora2019fine}
in order to show that two-layer neural networks with standard activation, trained by SGD from random initialization, can learn
analytic functions on the unit sphere. We then obtain our results by demonstrating how our representations of interest can be captured by analytic functions with power series representations of appropriately bounded norms. Formal statements and proofs for this section appear in Appendix~\ref{sec:analytic}. Let $S^{d}$ denote the unit sphere in $d$ dimensions.

\begin{theorem}\label{thm:analytic-informal}(Informal)
Given an analytic function $g(y)$,
the function $g(\bbet\cdot\x)$, for fixed
$\bbet\in\mathbb{R}^{d}$ (with $\bnorm\defeq\|\bbet\|_{2}$) and inputs $\x\in S^{d}$ is
learnable to error $\eps$ with
$n = O((\bnorm \tilde{g}'(\bnorm)+\tilde{g}(0))^2/\eps^2)$ examples using a single-hidden-layer,
finite width neural network of width ${\rm poly}(n)$ trained with SGD,
with
\begin{equation}
\tilde{g}(y) = \sum_{k=0}^{\infty}|a_k|y^k
\end{equation}
where the $a_k$ are the power series coefficients of $g(y)$.
\end{theorem}

We will refer to $\tilde{g}'(1)$ as the norm of the function $g$---this captures the Rademacher complexity of learning $g$,
and hence the required sample complexity. We also show that the $\tilde{g}$ function in fact tightly captures the Rademacher
complexity of learning $g$, i.e. there is a lower bound on the Rademacher complexity based on the coefficients of $\tilde{g}$
for certain input distributions (see Corollary \ref{cor:lower}  in Section \ref{sec:lower} in the appendix).

We also note that we can prove a much more general version for multivariate analytic functions $g(\x)$, with
a modified norm function $\tilde{g}(y)$ constructed from the multivariate power series representation of
$g(\x)$ (Theorem \ref{thm:multivar} in Appendix~\ref{sec:analytic}).
The theorems can also be extended to develop a ``calculus of bounds'' which lets us compute new bounds
for functions created via combinations of learnable functions. In particular, we have a product rule and a chain rule:

\begin{corollary}[Product rule]
\label{cor:prod_rule}
Let $g(\x)$ and $h(\x)$ meet the conditions of Theorem 1. Then the product
$g(\x)h(\x)$ is efficiently learnable as well, with $O(M_{g\cdot h}/\epsilon^2)$ samples where
\begin{equation}
\sqrt{M_{g\cdot h}} = \tilde{g}'(1)\tilde{h}(1)+\tilde{g}(1)\tilde{h}'(1)+\tilde{g}(0)\tilde{h}(0).
\end{equation}
\end{corollary}
\begin{corollary}[Chain rule]
\label{cor:chain_rule}
Let $g(y)$ be an analytic function
and $h(\x)$ be efficiently learnable, with auxiliary functions $\tilde{g}(y)$ and
$\tilde{h}(y)$ respectively. Then the composition $g(h(\x))$ is efficiently learnable as well with $O(M_{g\circ h}/\epsilon^2)$ samples where
\begin{equation}
\sqrt{M_{g\circ h}} = \tilde{g}'(\tilde{h}(1))\tilde{h}'(1)+\tilde{g}(\tilde{h}(0)),
\end{equation}
provided that $\tilde{h}(0)$ and $\tilde{h}(1)$ are in the radius of convergence of $g$.
%provided that $\tilde{g}(\tilde{h}(0))$ and $\tilde{g}(\tilde{h}(1))$ converge
%(equivalently, if $\tilde{h}(0)$ and $\tilde{h}(1)$ are in the radius of convergence of $g$).
\end{corollary}

The calculus of bounds enables us to prove learning bounds on increasingly expressive functions, and we can prove results that may be of independent interest. As an example, we show in Appendix~\ref{sec:learning_gravity} that forces on $k$ bodies interacting via Newtonian gravitation, as shown in Figure~\ref{fig:learn}, can be learned to error
$\epsilon$ using only $k^{O(\ln(k/\epsilon))}$ examples (even though the function $1/x$ has a singularity at 0).

\subsection{Task coding via clusters}\label{sec:clusters}
%\subsubsection{Multiple functions can be learnt by packing them into different clusters/subspace}

%\aga{Needs smoother transition; attempted below. Overall feel like this section could be a bit cleaner; many partially repeated points.} {\bf Tried to revise the first sentence to connect back to the previous section more explicitly --BJ}\aga{edited once more} I'm happy with this! --BJ

%In order to learn in the multitask setting, the inputs must encode both the task and the input for that specific task.
Our analysis of learning analytic functions
allows us to prove that a single network with standard training can learn multiple tasks.
We formalize the problem of learning multiple tasks as follows.
In general, these networks take pairs of inputs $(\code,\x)$ where $\code$ is a {\em task code} and $\x$ is the input (vector) for the 
chosen task represented by $\code$. 
%\VS{We had bold-face for vectors earlier, and now switch to normal font.} Switching to bold-face for all vectors for consistency --BJ 
We assume both $\code$ and $\x$ have fixed dimensionality. These pairs are then encoded by the 
concatenation of the two vectors, which we denote by $\code;\x$. Given $k$ tasks, corresponding to 
evaluation of functions $f_1,\ldots,f_k$ respectively
on the input $\x$, the $i$th task has a corresponding code $\code^{(i)}$. Now, we wish to learn a function $g$ such that $g(\code^{(i)};\x)=f_i(\x)$ 
for examples of the form $(\code^{(i)};\x,f_i(\x))$. This $g$ is a ``monolithic'' function combining the $k$ tasks. More generally, there may be some noise (bounded within a small ball around $\code^{(i)}$)  in the task codes which would require learning the monolithic function $g(\code,x) = f_{j} (\x)$ where $j=\text{argmin}_i \|\code-\code^{(i)}\|_2$ . 
Alternately the task-codes are not given explicitly but are inferred by checking which ball-center $\code^{(i)}$ (unique per task)  is closest to  the input $\x$ (see Fig. \ref{fig:functions} (left) for an example). Note that these are all generalizations of a simple one-hot coding.

\begin{figure}
\centering
    \includegraphics[width=0.8\textwidth]{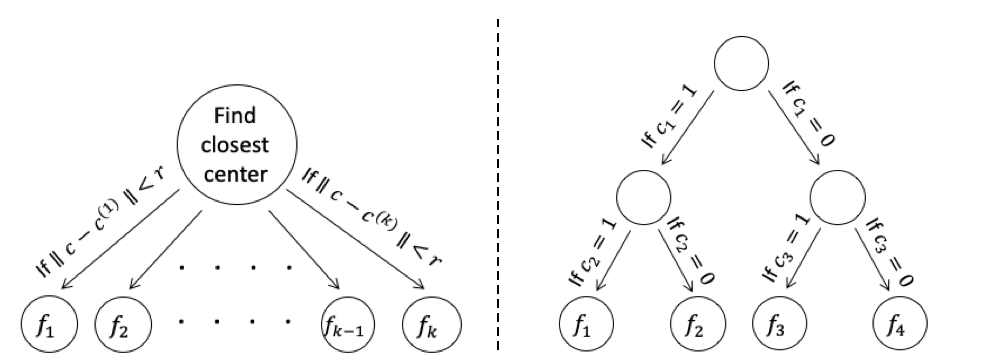}
    \caption{Some of the task codings which fit in our framework. On the left, we show a task coding via clusters. Here, $\code^{(i)}$ is the code for the $i$th cluster. On the right, we show a task coding based on low-depth decision trees. Here, $\code_{i}$ is the $i$th coordinate of the code $\code$ of the input datapoint.}
    \label{fig:functions}
\end{figure}

%The most basic type of task codings we consider check whether $\code$ lies near a specific point
%or subspace associated
%with a particular task. Thus, each of the $k$ tasks are encoded, respectively, by sufficiently 
%well-separated prototypical task code $\code^{(i)}$, and for any datapoint $(\code, \x)$ its task is given by the prototypical task codes $\code^{(i)}$ closest to  $\code$
%; that is, 
%$g(\code,x) = f_{argmin_i ||c-c^{i}||_2} (\x)$ . 
%or distinct subspaces $\A^{(i)}\code=0$.
%\aga{if not onerous, would like the $A_{i}$ to be bold as well. If too much work then don't worry.} Not hard at all. --BJ
%Note that this generalizes a simple one-hot coding. 

We assume throughout that the $f_i$ are analytic, with bounded-norm multinomial Taylor series representations. Our technical tool is the following Lemma
(proved in Appendix \ref{sec:analytic})
%\aga{where? Generally should have pointers to proofs} which shows that
which shows that the univariate step function $\mathbf{1}(x\ge 0)$ can be approximated with error $\eps$ and margin $\gamma$ using a low-degree polynomial which can be learnt using SGD.
\begin{lemma}\label{lem:indicator_poly_informal}
Given a scalar $x$, let  $$\Phi(x,\gamma,\eps)=(1/2)\left(1+ \erf\left({Cx\sqrt{\log(1/\eps)}}/{\gamma}\right) \right)$$ 
where $\erf$ is the Gauss error function and $C$ is a constant. Let ${\Phi'}(x,\gamma,\eps)$ be the function $\Phi(x,\gamma,\eps)$ with its Taylor series truncated at degree $O(\log(1/\eps)/\gamma)$. Then,
\[ 
{\Phi'}(x,\gamma,\eps) =\begin{cases} 
       O(\eps) & x\le -\gamma/2, \\
      1-O(\eps) & x \ge \gamma/2  . 
   \end{cases}
\]
Also, ${\Phi'}(x,\gamma,\eps)$ can be learnt using SGD with at most $e^{O((\log(1/\eps)/\gamma^2))}$ examples.
\end{lemma}
Using this lemma, we show
that indicator functions for detecting membership in a ball near a prototype $\code^{(i)}$ can also be sufficiently 
well approximated by functions with such a Taylor series representation.
Specifically,
we use the truncated representation of the $\erf$ 
function to indicate that $\norm{\code-\code^{(i)}}$ is small. As long as the centers are sufficiently well-separated, we can find a low-degree, low-norm function this way using Lemma \ref{lem:indicator_poly_informal}. For example, to check if $\code$ is within distance $r$ of center $\code^{(i)}$ we can use $\mathbf{1}(\norm{\code-\code^{(i)}}^2\le r^2)$, which can be approximated using the $\phi'$ function in Lemma \ref{lem:indicator_poly_informal}.
 Then given such 
approximate representations for the task indicators $I_1(\code),\ldots,I_k(\code)$, the function 
$g(\code;\x)=I_1(\code)f_1(\x)+\cdots+I_k(\code)f_k(\x)$ has norm linear in the complexities of the task functions, so that they are 
learnable by Theorem~\ref{thm:analytic-informal} (we scale to inputs to lie within the unit ball as required by Theorem~\ref{thm:analytic-informal}).
%\VS{which theorem to cite here?}. Executive decision: we use the informal statement in the previous section, so it's self-contained. --BJ
We state the result below, for the formal statement and proof see Appendix \ref{sec:clusterfunctions}.

%[[*****state for separated balls per task too where no explicit task codes are given:Rina****]]
%Packing multiple functions into one Network:
%To illustrate this let us look at k different analytical functions (that can be represented using multinomial Taylor series with bounded norm on the coefficients). We can embed these functions into into different clusters that are far apart. If each $f_i$ is from $R^d->R$, we can pack all these functions into a single function $g$ where to evaluate $f_i(x)$ we simply evaluate $g(c_i ; x)$  where $c_i$ is a vector and ";" represents the concatenation operator that represents the cluster center of the $i$th function. Alternately We can also pack them into clusters by doing $g(c_i + x/k)$ -- the scaling by $k$ keeps the clusters well separated. The above method can be used to pack $k$ functions $f_i$ where the complexity of $g$ is the sum of the complexities of the individual functions.

\begin{theorem}\label{thm:informa_clusterfunctions}(Informal)
Given $k$ analytic functions having Taylor series representations with norm at most $\poly(k/\eps)$ and degree at most $O(\log(k/\eps))$, a two-layer neural network trained with SGD can learn the following functions $g(\code;\x)$ on the unit sphere to accuracy $\eps$ with sample complexity  $poly(k/\eps)$ times the sum of the sample complexities for learning each of the individual functions:
\begin{compactitem}
%\item for $\Omega(1)$-far subspaces $\A^{(1)},\ldots,\A^{(k)}$, if $\A^{(i)}\code=0$ then $g(\code;\x)=f_i(\x)$
\item for $\Omega(1)$-separated codes $\code^{(1)},\ldots,\code^{(k)}$, if $\|\code-\code^{(i)}\|_2\leq O(1)$, then $g(\code;\x)=f_i(\x)$.
%\item for $\Omega(1)$-separated $\code^{(1)},\ldots,\code^{(k)}$, if $\code$ is in a small $\ell_2$ ball around $\code^{(i)}$, then $g(\code;\x)=f_i(\x)$.
\end{compactitem}
%By packing $k$ functions into different clusters/subspaces, a neural network with one hidden layer can learn all of the functions to accuracy $\eps$ with sample complexity  $poly(k/\eps)$ times the sum of the sample complexities for learning each of the individual functions.
\end{theorem}

%\emph{Note that the $\code^{(i)}$ need to be at least $10\log k$ coordinates so that subspaces of the cluster centers are far apart. In fact if we use exactly $\log k$ then it is provably hard, which is shown in section xxxx.}\VS{do we need these lines?} Doesn't seem like it at the moment. So I am cutting it. --BJ

\subsection{Task coding via low-depth decision trees}

Theorem~\ref{thm:informa_clusterfunctions} can be viewed as performing a single $k$-way branching choice of which task function to evaluate. Alternatively, we can consider a sequence of such choices, and obtain a {\em decision tree} in which the leaves indicate which task function is to be applied to the input. We first consider the simple case of a decision tree when $\code$ is a $\{\pm 1\}$-valued vector. We can check that the values $c_1,\ldots,c_h$ match the fixed assignment $c^{(i)}_1,\ldots,c^{(i)}_h$ that reaches a given leaf of the tree using the function $I_{\code^{(i)}}(\code)=\prod_{j=1}^h\frac{c_j+c^{(i)}_j}{2}$ (or similarly for any subset of up to $h$ of the indices). Then $g(\code;\x)=I_{\code^{(1)}}(\code)f_1(\x)+\cdots+I_{\code^{(k)}}(\code)f_k(\x)$ represents our decision tree coding of the tasks (see Fig. \ref{fig:functions} (right) for an example). For  the theorem, we again scale the inputs to lie within the unit ball: 
%Note that the above packing can be viewed as a decision tree that branches $k$-ways depending on how close the input $x$ is to each center $h_i$. Alternatively we can consider a decision tree with multiple levels where each node is branching based on a coordinates' value, with a separate function applied to the input at every leaf node. This is an alternate way of packing $k$ functions.

%Correspondence to learning decision trees:
%We prove Theorem 1: 

\begin{theorem}\label{thm:decisiontree_informal}
(Informal) Two-layer neural networks trained with SGD can learn such a decision tree with depth $h$ within error $\eps$  with sample complexity $O(d^{h}/\eps^2)$ times the sum of the sample complexity for learning each of the individual functions at the leaves. Furthermore, conditioned on the hardness of learning parity with noise, $d^{\Omega(h)}$ examples are in fact necessary to learn a decision tree of depth $h$.
\end{theorem}

We can generalize the previous decision tree to allow a threshold based decision at every internal node, instead of just looking at a coordinate. Assume that the input data lies in the unit ball and that each decision is based on a margin of at least $\gamma$. We can then use a product of our truncated $\erf$ polynomials to represent branches of the tree. We thus show: %(for formal statement and proofs, see Section \ref{sec:decision_app} in the Appendix): 

%Decision trees:
%If we are given a decision tree where each decision of depth $h$ is based on a discrete binary variables taking values in $\{-1,+1\}$ then such a function can be learnt in time $d^h$ and this is tight (we prove a matching lower bound based on parity with noise Theorem \ref{thm:parity})

\begin{theorem}\label{thm:decisiontreemargin_informal}
(Informal)
If we have a decision tree of depth $h$ where each  decision is based on a margin of at least $\gamma$, then we can learn such a such a function within error $\eps$ with sample complexity $e^{O(h\log(1/\eps)/\gamma^2)}$ times the sample complexity of learning each of the leaf functions.
\end{theorem}

For the formal statements and proofs, see Appendix \ref{sec:decision_app}. 
Note that by Theorem \ref{thm:decisiontree_informal}, the exponential dependence on the depth in these theorems is necessary.%: we show a corresponding lower bound in Section \ref{sec:lower} in the Appendix.

\subsection{Simple programming constructs}

So far, we have discussed jointly learning $k$ functions with task codings represented by clusters and decision trees. We now move to a more general setup, where we allow simple programming constructs such as compositions, aggregation, concatenation, and branching  on different functions. At this stage, the distinction between ``task codes'' and ``inputs'' becomes somewhat arbitrary. Therefore, 
%as we are no longer evaluating one of $k$ distinct functions 
we will generally drop the task codes $\code$ from the inputs. The class of programming constructs  we can learn is a generalization of the decision tree and we refer to it as a \emph{generalized decision program}. 

\begin{definition}
We define a \emph{generalized decision program} to be a circuit with fan-out 1 (i.e., a tree topology). Each gate in the circuit computes a function of the outputs of its children, and the root (top) node computes the final output. All gates, including the leaf gates, have access to the input $\x$. 

\end{definition}

We can learn generalized decision programs where each node evaluates one among a large family of operations, first described informally below, and then followed by a formal definition.

\paragraph{Arithmetic/analytic formulas}
As discussed in Section \ref{sec:learning_analytic}, learnability of analytic functions not only allows us to learn functions with bounded Taylor series, but also sums, products, and ratios of such functions. Thus, we can learn constant-depth arithmetic formulas with bounded outputs and analytic functions (with appropriately bounded Taylor series) applied to such learnable functions.
%We first observe that the sum of two analytic functions with bounded Taylor representations yields a function with the same degree and similar norm, and hence is learnable in essentially the same sample complexity. Likewise, the product of two such functions yields a function in which the degrees sum and the norms are polynomially related. Finally, the composition of such functions yields a function in which the degrees multiply and the norm increases polynomially. Thus, we observe that we can learn constant-depth arithmetic formulas with bounded constants and analytic functions (with appropriately bounded Taylor representations) applied to such learnable functions. \VS{This can be shortened now.}
%Note that in the above decision tree we can also handle other operators such as sum nodes with essentially the same sample complexity. We can also get an analytic representation if the tree involves product node by multiplying the analytic representations for the individual functions. More generally, we can apply any analytic transform over the child nodes. We will now generalize this to a programming language view of a decision program.

\paragraph{Aggregation}

We observe that the sum of $k$ functions with bounded Taylor representations yields a function of the same degree and norm that is at most $k$ times greater; the average of these $k$ functions, meanwhile does not increase the magnitude of the norm. Thus, these standard aggregation operations are represented very efficiently. These enable us to learn functions that answer a family of SQL-style queries against a fixed database as follows: suppose $I(\x,r)$ %\aga{task codes were used here after promising to get rid of them!} Yeah, OK. removed for consistency. --BJ
is an indicator function for whether or not the record $r$ satisfies the predicate with parameters $\x$. Then a sum of the $m$ entries of a database that satisfy the predicate given by $\x$ is represented by $I(\x,r^{(1)})r^{(1)}+\cdots+I(\x,r^{(m)})r^{(m)}$. Thus, as long
as the predicate function $I$ and records $r^{(i)}$ have bounded norms, the function mapping the parameters $\x$ to the result of the query is learnable. We remark that max aggregation can also be represented as a sum of appropriately scaled threshold indicators, provided that there is a sufficient gap between the maximum value and other values.

%Functions returning objects:
\paragraph{Structured data}
We note that our networks already receive vectors of inputs and may produce vectors of outputs. Thus, one may trivially structured inputs and outputs such as those in Fig. \ref{fig:learn} (right) using these vectors. %But, the different task functions may operate on different kinds of data structures that operate on different fields, which by virtue of the monolithic architecture, must share the same set of input and output wires in the network. If we regard the wiring as fixed, this may be a concern. But here, we note that this concern may be resolved if the network operates on a ``sketched'' representation of a large space of labeled fields.
We now formalize this by defining the class of functions we allow. 

\begin{definition}
 We support the following operations at any gate in the generalized decision program. Let every gate have at most $k$ children. Let $g$ be the output of some gate and $\{f_1,\dots,f_k\}$ be the outputs of the children of that gate.

\begin{tightenumerate}
    \item Any analytic function of the child gates which can be approximated by a polynomial of degree at most $p$, including sum $g=\sum_{i=1}^{k} f_i$ and product of $p$ terms $g=\Pi_{i=1}^p f_i$.
    \item Margin-based switch (decision) gate with children $\{f_1,f_2\}$ and some constant margin $\gamma$, i.e., $g=f_1  \text{ if }\ip{\bbet}{\x}-\alpha\leq -\gamma/2,$ and $g=f_2 \text{ if }\ip{\bbet}{\x}-\alpha\geq \gamma/2$, for a vector $\bbet$ and constant $\alpha$.
    \item Cluster-based switch gate with $k$  centers $\{\code^{(1)},\dots,\code^{(k)}\}$, with separation $r$ (for some constant $r$), i.e. the output is $f_i$ if $\norm{\x-\code^{(i)}}\le r/3$. A special case of this is a look-up table which returns value $v_i$ if $\x=\code^{(i)}$, and 0 if $\x$ does not match any of the centers.
    \item Composition of two functions,  $g(\x)=f_1(f_2(\x))$.
    \item Create a tuple out of separate fields by concatenation: given inputs $\{f_1,\dots, f_k\}$ $g$ outputs a tuple $[f_1,\dots, f_k]$, which creates a single data structure out of the children. Or, extract a field out of a tuple: for a fixed field $i$, given the tuple $[f_1,\ldots,f_k]$, $g$ returns $f_i$.  
    \item For a fixed table $T$ with $k$ entries $\{r_1,\dots, r_k\}$, a Boolean-valued function $b$, and an analytic function $f$,  SQL queries of the form \verb|SELECT SUM f(r_i), WHERE b(r_i, x)| for the input $\x$, i.e., $g$ computes $\sum_{i:b(r_i,\x)=1}f(r_i)$. (We assume that $f$ takes bounded values and $b$ can be approximated by an analytic function of degree at most $p$.) For an example, see the function \verb|avg_income_zip_code()| in Fig.~\ref{fig:learn} (right).
    \end{tightenumerate}

\end{definition}

%Thus type information and objects can be interpreted in deep networks by encoding types and object-structures into subspaces. Thus no explicit type support may be needed and dynamic types may automatically “arise” in networks as new subclusters and subspaces. Different functions, modules, classes may not need different physical networks and may all be concurrently implemented by the same network but may pack their data flow into different subspaces of the input/outputs of the layers of these networks.

As an example of a simple program we can support, refer to Fig.~\ref{fig:learn} (right) which involves table lookups, decision nodes, analytic functions such as Euclidean distance, and SQL queries. Theorem \ref{thm:program_informal} is our learning guarantee for  generalized decision programs. See Section \ref{sec:program_app} in the Appendix for proofs, formal statements, and a detailed description of the program in Fig.~\ref{fig:learn} (right).

\begin{theorem}\label{thm:program_informal}
(Informal)
Any generalized decision program of constant depth $h$ using the above operations with $p\le O(\log (k/\eps))$  can be learnt within error $\eps$ with sample complexity $k^{\poly(\log (k/\eps))}$. For the specific case of the program in Fig. \ref{fig:learn} (right), it can be learnt using $(k/\eps)^{O(\log (1/\eps))}$ examples, where $k$ is the number of individuals in the database. 
\end{theorem}

\section{Experiments}\label{sec:case1exp}

We next empirically explore the learnability of multiple functions by a two layer neural network when the tasks are coded by well-separated clusters or decision trees, and more generally the learnability of SQL-style aggregation for a fixed database. We find good agreement between the empirical performance and the bounds of Section \ref{sec:techinical_overview}. See Appendix \ref{sec:appendix_app} for more details of the experimental setup.

%We demonstrate the usefulness of the theoretical  results in Section \ref{sec:techinical_overview} and explore the learnability of multiple functions 1`Specifically, the sample complexity bounds [[cross-ref needed]] agrees with experiment results.

\paragraph{Learning binary classification for well-separated clusters data} We demonstrate through experiments
on synthetic data that a single neural network can learn multiple tasks if the tasks are well-separated into clusters, as we discussed in Section \ref{sec:clusters}.  Here the data is drawn from a mixture of $k$ well-separated Gaussians in $d=50$ dimensions. Within each Gaussian, the data points are marked with either of two labels. For the label generation, we consider two cases, first when the labels within each cluster are determined by a simple linear classifier, and second when the labels are given by a random teacher neural network with one hidden layer of $10$ hidden units. Fig. \ref{fig:nn} shows the performance of a single two-layer neural network with $50 k$ hidden units on this task. The performance of the neural network changes only slightly on increasing the number of clusters ($k$), suggesting that a single neural network can learn across all clusters. 
%{\bf We should say how many hidden units we use for the monolithic NN. --BJ} \xin{Done.}

\begin{figure}
\centering
\begin{subfigure}{0.48\textwidth}
    \centering
    \includegraphics[width=\textwidth]{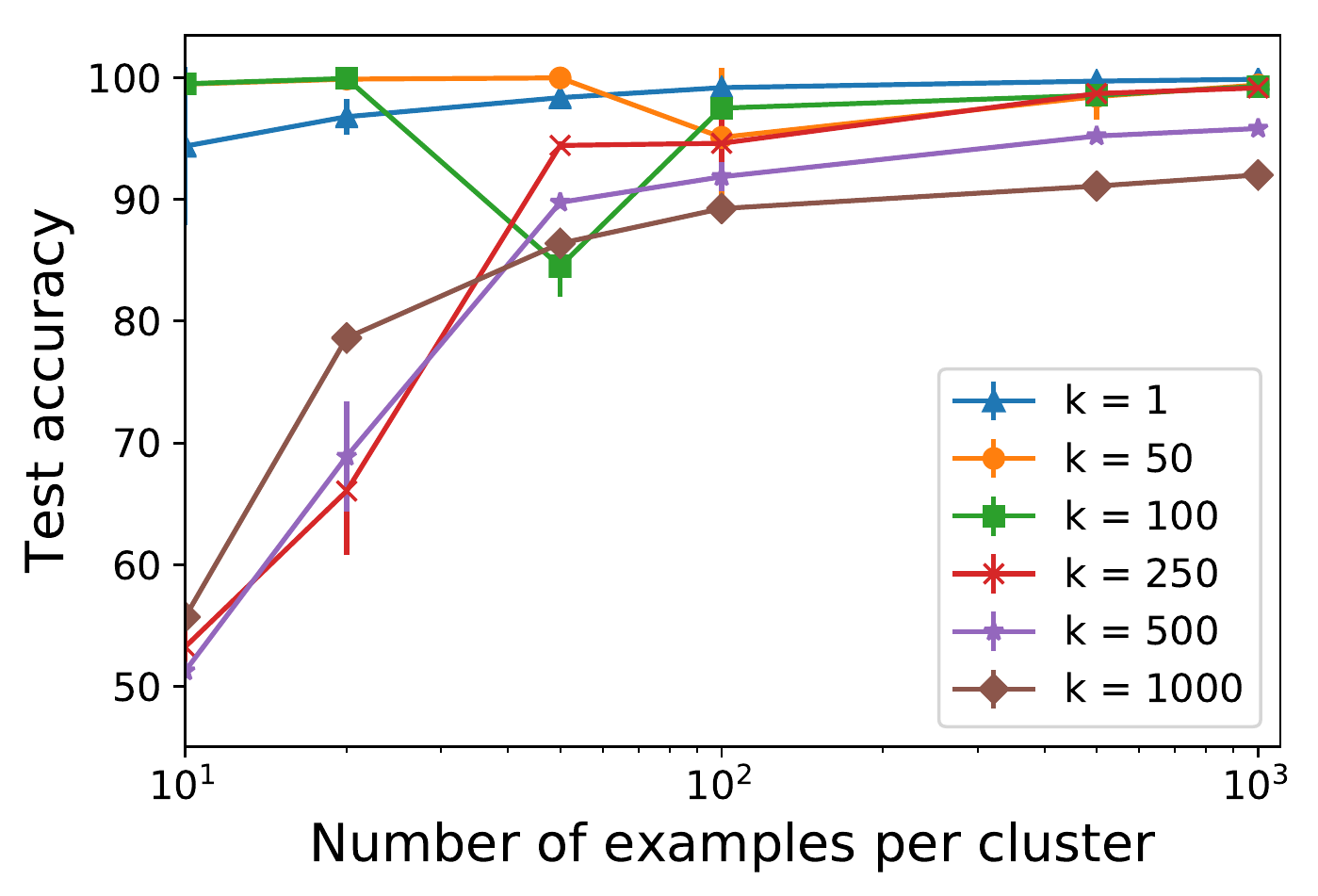}
    \caption{Random linear classifier for each cluster.}
    \label{fig:nn1}
\end{subfigure}
\begin{subfigure}{0.48\textwidth}
    \centering
    \includegraphics[width=\textwidth]{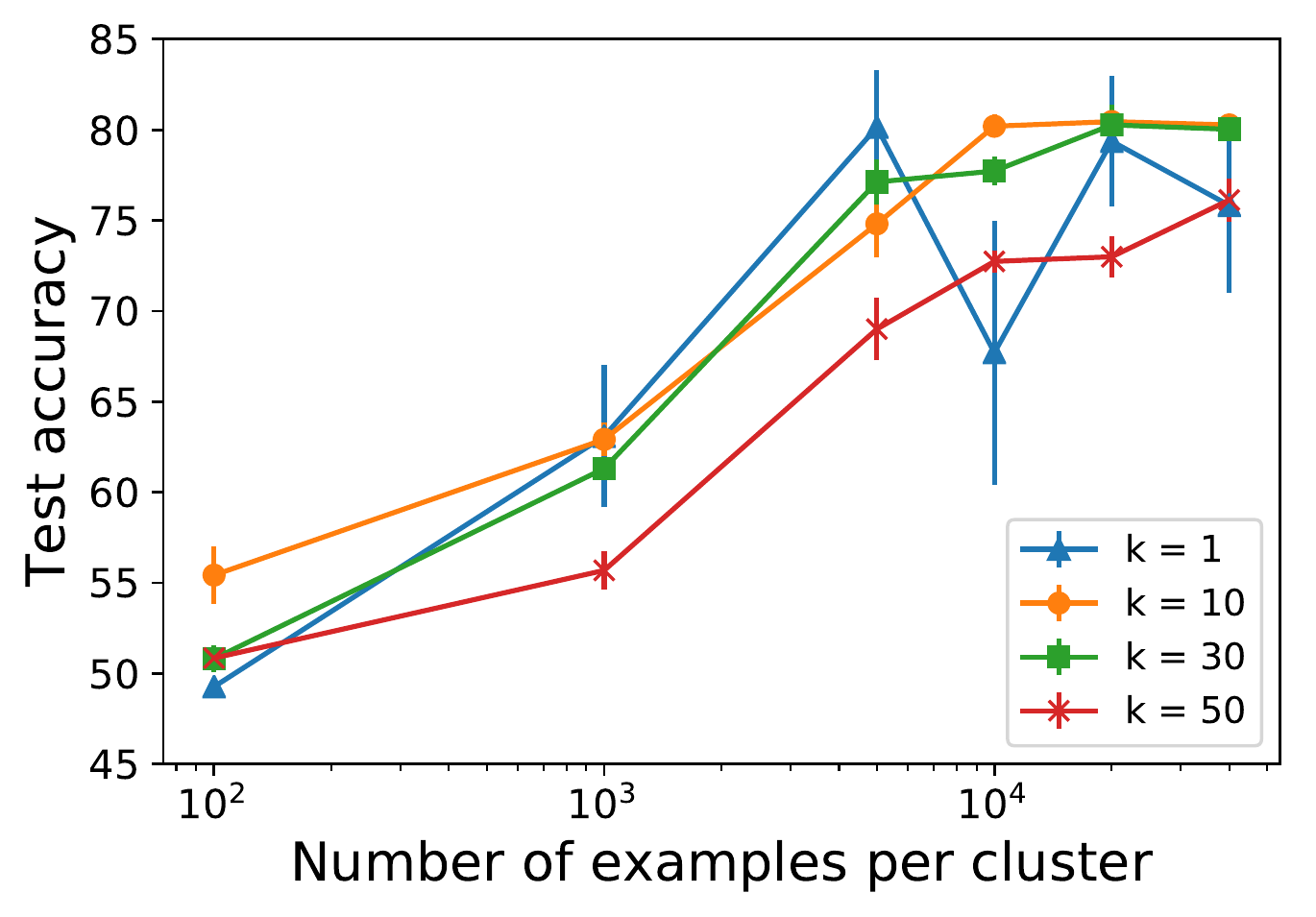}
    \caption{Random teacher network for each cluster.}
    \label{fig:nn2}
\end{subfigure}
\caption{Binary classification on multiple clusters, results are an average over 3 trials. A single neural network does well even when there are multiple clusters. The error does not increase substantially on increasing the number of clusters $k$}
\label{fig:nn}
\end{figure}

\paragraph{Learning polynomial functions on leaves of a decision tree} We consider the problem of learning polynomial functions selected by a decision tree. The data generation process is as follows. We first fix parameters: tree depth $h$, decision variable threshold margin $\gamma$, number of variables $k$, and degree $p$ for leaf functions. Then we specify a full binary decision tree of depth $h$ with a random polynomial function on each leaf. To do this, we first generate thresholds $t_1, t_2, ..., t_h$ from the uniform distribution on $[0, 1]$ and $2^h$ leaf functions which are homogeneous polynomials of $k$ variables and degree $p$, with uniformly distributed random coefficients in $[0, 1]$. A train/test example $(\x, y)$ where $\x = (x_1, ..., x_h, x_{h+1}, ..., x_{h+p})$ is generated by first randomly sampling the $x_i$'s from the uniform distribution on $[0, 1]$, selecting the corresponding leaf based on $x_1, ..., x_h$ (that is, go left at the first branch if $x_1 \leq t_1$,  otherwise go right, etc), and computing $y$ by evaluating the leaf function at $(x_{h+1}, ..., x_{h+p})$. The data is generated with the guarantee that each leaf has the same number of data points. Fig.~\ref{fig:decision_tree} shows the performance of a two-layer neural network with $32 \times 2^{h}$ hidden units, measured in the R-squared metric. Here the R-squared metric is defined as $1 - \sum_{i} (\hat{y}_i - y_i)^2/\sum_{i} (y_i - \overline{y})^2$, and is the fraction of the underlying variance explained by the model. Note that for a model that outputs the mean $\overline{y}$ for any input, the R-squared metric would be zero. We observed for a fixed number of training samples, accuracy increases as threshold margin increases, and the dependence of sample complexity on test error agrees with the bound in Theorem \ref{thm:decisiontreemargin_informal}.

%We observe if the training set size is scaled linearly with number of leaves, then the test error does not change much when the depth of the tree varies. Also, increasing the threshold margin decreases the test error. Both observation are in accordance with the $e^{h/\gamma^2}$ sample complexity bound in Theorem \ref{thm:decisiontreemargin_informal}. 

\begin{figure}
\centering
\begin{subfigure}{0.48\textwidth}
    \centering
    \includegraphics[width=\textwidth]{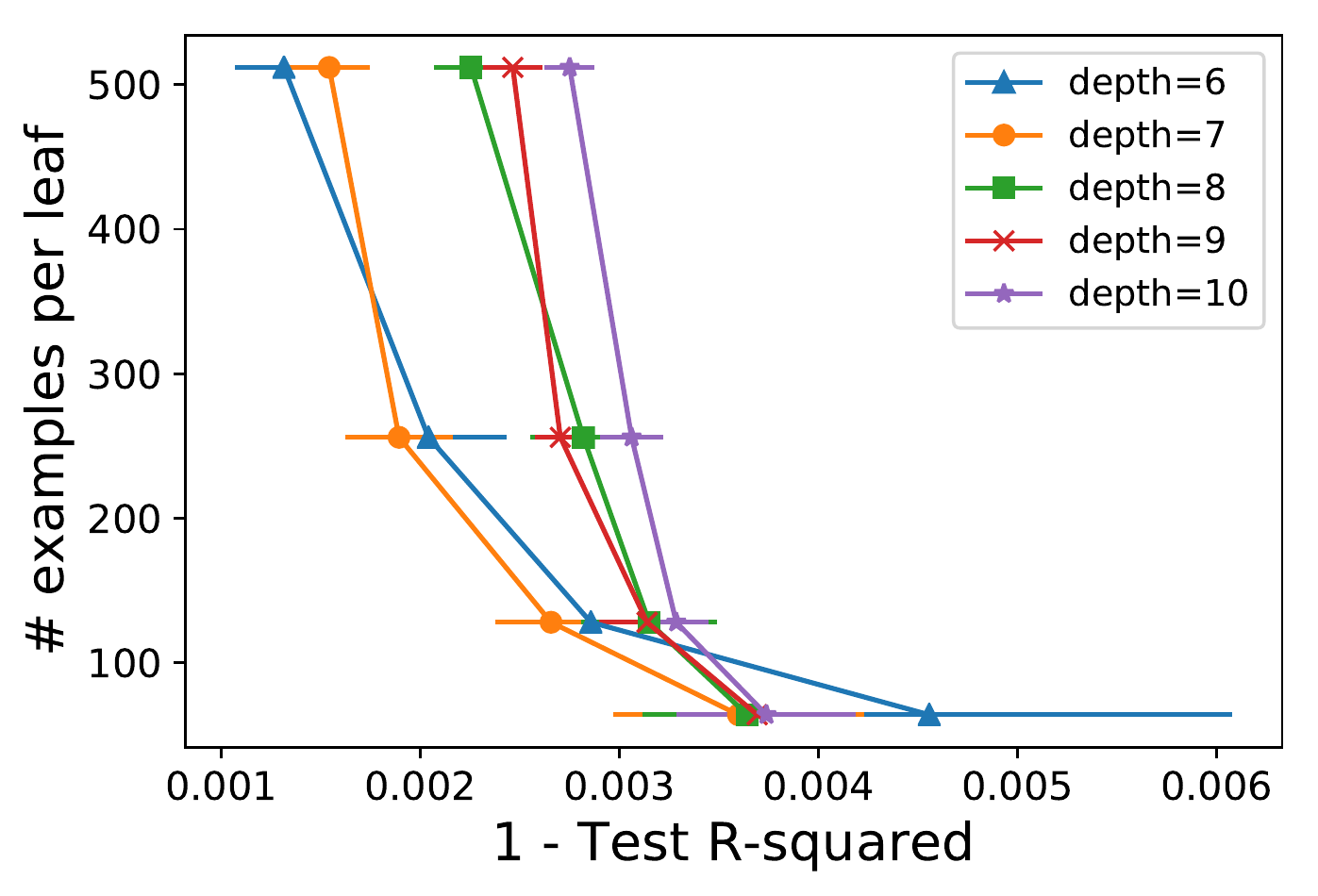}
    \caption{Fixed threshold margin $\gamma=0.1$.}
    \label{fig:decision_tree_1}
\end{subfigure}
\begin{subfigure}{0.48\textwidth}
    \centering
    \includegraphics[width=\textwidth]{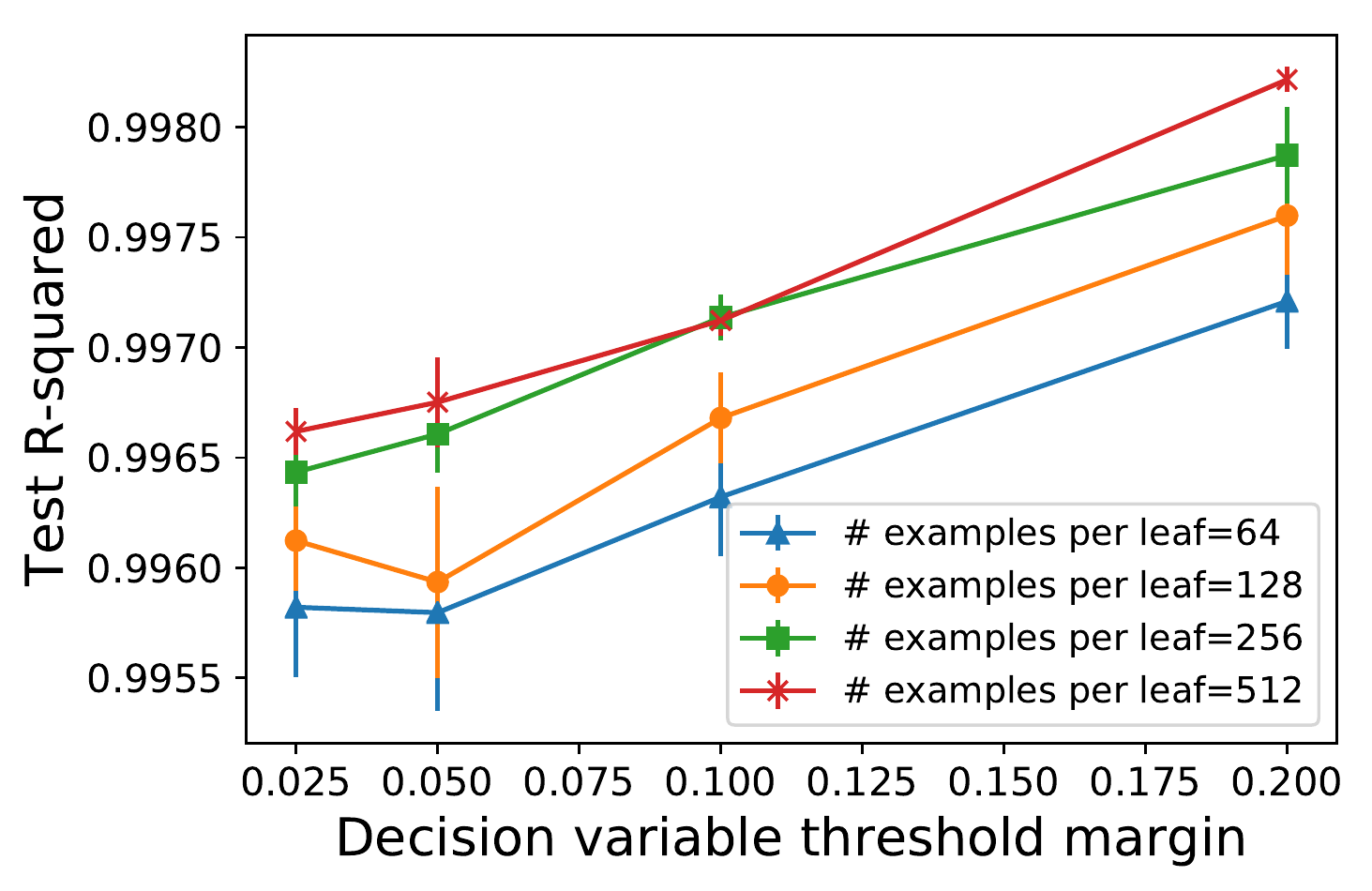}
    \caption{Fixed tree depth $h = 10$.}
    \label{fig:decision_tree_2}
\end{subfigure}
\caption{Learning random homogeneous polynomials of $4$ variables and degree $4$ on the leaves of a
decision tree, the results are averaged over 7 trials. (a) Sample complexity scales as $e^{O(h\log(1/\eps)/\gamma^2)}$ with error $\epsilon$, where error is measured by (1-Test R-squared). (b) For fixed tree depth, accuracy increases with increasing margin. } %\aga{descriptions of (a) and (b) should be main takeaways from plot. %For a fixed number of training examples per leaf node, test error for decision trees with different depth are similar. The decision tree has depth $6, 7, 8, 9$ or $10$, threshold margin $0.1$.
\label{fig:decision_tree}
\end{figure}

%The tree has $2^h$ leaves, and the polynomial functions on the leaves are homogeneous polynomials of $k$ variables and degree $p$ with uniformly distributed random coefficients in $[0, 1]$. This fully specifies a binary decision tree with polynomial leaf functions. Train and test data are then generated with the guarantee that each leaf has the same number of data points. For more details, see the appendix.

\paragraph{Learning SQL-style aggregation queries} We demonstrate the learnability of SQL-style aggregation queries, which are functions of the form \verb|SELECT SUM/MIN/MAX f(x)|  \verb|WHERE p(x) from DATABASE|. The train and test datasets are generated from the Penn World Table dataset \citep{feenstra2015next}, which contains $11830$ rows of economic data. The \verb|WHERE| clause takes the form of $(x_{i_1} \geq t_{i_1}$) \text{AND} \ldots \text{AND} ($x_{i_k} \geq t_{i_k})$, where $x_{i_1},\ldots, x_{i_k}$ are $k$ randomly selected columns and $t_{i_1},\ldots, t_{i_k}$ are randomly selected values from the columns. The query target function is randomly selected from \verb|SUM|, \verb|MAX|, and \verb|MIN| and is over a fixed column (\verb|pl_x| in the table, which stands for price level for imports). The R-squared metric for a two-layer neural network with $40$k hidden units is summarized in Table~\ref{tab:sql_aggregation}. We observe that a neural network learns to do SQL-style aggregation over dozens of data points, and for a fixed database, the test error only varies slightly for different numbers of columns in the \verb|WHERE| clause. % For more details, see the appendix. 

\begin{table}[h]
    \centering
    \caption{R-Squared for SQL-style aggregation. A single network with one hidden layer gets high R-Squared values, and the error does not increase substantially if the complexity of the aggregation is increased by increasing the number of columns in the WHERE clause.}
    \label{tab:sql_aggregation}
    % neurips style file is asking that table headings be above table
    \resizebox{\columnwidth}{!}{%
    \begin{tabular}{@{}lccccc@{}}
    \toprule
    \# columns in WHERE clause & 1 & 2 & 3 & 4 & 5\\
    \midrule
%    Average \# data points & 151 &  134 & 103 & 71 & 39 \\
%    \midrule
    Median \# data points & 21 &  12 & 9 & 4 & 3 \\
    \midrule
    Test R-Squared    & ($93.31 \pm 0.11$) \% &  ($93.01 \pm 2.7$)\% & ($91.86 \pm 2.59$) \% & ($94.84 \pm 1.86$) \% & ($92.51 \pm 2.2$) \% \\
    \bottomrule
    \end{tabular}
    }
    \vspace{2pt}
\end{table}

%\vspace{-5pt}
%\subsection{Theoretical analysis}

%It is challenging to that that neural networks can provably learn multiple functions because we do not have a good understanding of the functions that neural networks can learn. However, we show that neural networks can provably learn multiple functions from a certain class, using a recent framework of \citet{arora2019fine}.

%\citet{arora2019fine} recently showed that certain classes of smooth functions can be learnt by gradient descent, despite over-parameterization. We show that if the data points corresponding to each task individually satisfy the property in \citet{arora2019fine}, then the union of those data points also satisfies the property in \citet{arora2019fine} and hence can be learnt by a single neural network---as long as the data points from the different tasks are orthogonal.

%\begin{restatable}{thm}{mult}
%\label{thm:mult}
%If the data-points for different tasks are orthogonal and if the tasks individually satisfy the property in \citet{arora2019fine}, then the overall problem satisfies the property in \citet{arora2019fine} and can be learnt by a single neural network with sample complexity equal to the sum of the sample complexity of learning the individual tasks. 

%\end{restatable}	

%\vspace{-5pt}

\section{Conclusion and Future Work}

Our results indicate that even using a single neural network, we can still learn tasks across multiple, diverse domains.
However, modular architectures may still have benefits over monolithic ones: they might use less energy
and computation, as only a portion of the total network needs to evaluate any data point. They may also be more interpretable,
as it 
is clearer what role each part of the network is performing. It is an open question if any of these benefits
of modularity can be extended to monolothic networks.
For instance, is it necessary for a monolithic network to have modular parts which perform identifiable simple computations? And if so, can we efficiently identify these from the larger network? This could help in 
interpreting and understanding large neural networks. %\aga{Didn't understand these last two points - are you saying the monolithic network is secretly modular, but we just didn't have to bake than in?} {\bf Yes. The questions are (1) whether this is necessarily true, and (2) whether we can spot this somehow. Revised this, see if it's clearer. --BJ}

Our work also begins to establish how neural networks can learn
functions which are represented as simple programs.
This 
perspective raises the question, how rich can these programs
be? Can we learn programs from a full-featured language? In particular, 
supposing that they combine simpler programs using other basic operations 
such as composition, can such libraries of tasks be learned as well, i.e., can these learned programs be reused? 
We view this as a compelling direction for future work.

\iffalse
\textcolor{red}{remove below?}
Though we primarily focused on two-layer neural networks with ReLU activations, an open question is what models
are optimal for different tasks and task combinations. For example, our results are applicable
even for learning using kernel methods since a two-layer infinite width network can be 
viewed as a kernel function. In Appendix \aga{ref} we show that for learning the gravity function
the NTK kernel for ReLU 
activation has superior learning guarantees than the Gaussian kernel. Our theory may enable a more detailed analysis
of popular models on tasks of interest. \VS{maybe move this to 2.1?}
\fi

\subsubsection*{Acknowledgements}
Brendan Juba was partially supported by NSF Awards CCF-1718380, IIS-1908287, and IIS-1939677, and was visiting Google during a portion of this work. Vatsal Sharan was supported in part by NSF award 1704417.

%\section*{Broader Impact} 

%As this work is largely theoretical, this work does not present any forseeable societal consequences. Our experiments validate our theoretical predictions, but do not provide state-of-the-art results on any tasks of consequence.
%% using paragraph heading so that it occupies less space...
% Broader impact is allowed to go on the 9th page.

\bibliographystyle{iclr2021_conference}
\bibliography{refs.bib}

% NOTE: I have moved the appendix to supplementary_main.tex

\newpage
\appendix

\section{Theoretical Results}

\subsection{Kernel learning bounds}
\label{sec:kernel_learning_bounds}

In this section, we develop the theory of learning analytic functions. For a given function $g$, we define
a parameter $M_{g}$ related to the sample complexity of learning $g$ with small error with respect to a given
loss function:

\begin{definition}
\label{def:eff_learnable}
Fix a learning algorithm, and a 1-Lipschitz loss function $\Lo$.
For a function $g$ over a distribution of inputs $\mathcal{D}$, a given error scale $\eps$, and a confidence parameter $\delta$, let the {\em sample complexity} $n_{g,\mathcal{D}}(\eps,\delta)$ be the smallest integer such that when the algorithm is given $n_{g,\mathcal{D}}(\eps,\delta)$ i.i.d.\ examples of $g$ on $\mathcal{D}$,
with probability greater than $1-\delta$, it produces a trained model
$\hat{g}$ with generalization error $\expect_{\x\sim\mathcal{D}}[\Lo(g(\x),\hat{g}(\x))]$ less than $\eps$. 
Fix a constant $C>0$. We say $g$ is \emph{efficiently learned} by the algorithm (w.r.t.\ $C$) if there exists a constant $M_g$ (depending on $g$) such that for all $\eps$, $\delta$, and distributions $\mathcal{D}$ on the inputs of $g$, 
$n_{g,\mathcal{D}}(\eps,\delta) \leq C([M_g+\log(\delta^{-1})]/\eps^2)$.
\end{definition}

For example, it is known (\cite{talagrand1994sharper}) that there exists a suitable choice of $C$ such that empirical risk
minimization for a class of functions efficiently learns those functions with $M_g$ at most the VC-dimension of that 
class.

\iffalse
Old definition below----
\begin{definition}
\label{def:eff_learnable}
Given a learning algorithm, we say that a function $g$ over a distribution of inputs $\mathcal{D}$
is \emph{efficiently learnable} if, given an error scale $\eps$, with probability greater than $1-\delta$,
the generalization error $\expect_{\x\sim\mathcal{D}}[\Lo(g(\x),\hat{g}(\x))]$ of the trained model
$\hat{g}$ with respect to any 1-Lipschitz loss function $\Lo$
is less than $\eps$ when the training data consists of at least $O([M_g+\log(\delta^{-1})]/\eps^2)$
i.i.d. samples drawn from $\mathcal{D}$, for some $n$-independent constant $M_g$.
\end{definition}
\fi

Previous work focused on computing $M_{g}$, for functions defined on the unit sphere, for wide neural
networks trained with SGD. We extend the bounds derived in \cite{arora_finegrained_2019} to analytic
functions, and show that they apply to kernel learning methods as well as neural networks.

The analysis in \cite{arora_finegrained_2019} focused on the case of training the hidden layers
of wide networks with SGD. We first show that these bounds are more general and in particular
apply to the case where only the
final layer weights are trained (corresponding to the NNGP kernel in \cite{lee_wide_2019}),
and therefore our results will apply to general kernel learning as well.
The proof strategy
consists of showing that finite-width networks have a sensible infinite-width limit, and showing that
training causes only a small change in parameters of the network.

Let $m$ be the number of hidden units, and $n$ be the number of data points.
Let $\y$ be the $n \times 1$ dimensional vector of training outputs.
Let $\h$ be a $n \times m$ random matrix denoting the activations of the hidden layer (as a function of the weights
of the lower layer) for all $n$ data points. We will first show the following:

\begin{theorem}\label{thm:kernel_bound}
For sufficiently large $m$, a function $g$ can be learned efficiently in the sense of
Definition \ref{def:eff_learnable} by training the final layer weights only with SGD,
where the constant $M_{g}$ given by
\begin{equation}
M_{g} \leq \y^{\tpose}(\H)^{-1}\y
\end{equation}
where we define $\H$ as
\begin{equation}
 \H = \expect[\h\h^{\tpose}]   
\end{equation}
which is the NNGP kernel from \cite{lee_wide_2019}.
\end{theorem}

We require some technical lemmas in order to prove the theorem. We first need to show that
$\H$ is, with high probability, invertible.
If $K(\x,\x')$, the kernel function which generates $\H$
is given by a infinite Taylor series in $\x\cdot\x'$ it can be argued that 
$\H$ has full rank for most real world distributions. For example, the ReLU activation this holds as long as no two data 
points are co-linear (see Definition 5.1 in \cite{arora_finegrained_2019}). We can prove this more explicitly in the
following lemma:

\begin{lemma}
If all the $n$ data points $x$ are distinct and the Taylor series of  $K(\x,\x')$ in $\x\cdot\x'$ has positive coefficients everywhere then $\H$ is not singular.
\end{lemma}
\begin{proof}
First consider the case where the input $x$ is a scalar.
Since the Taylor series 
corresponding to $K(x,x')$ consists of monomials of all degrees of
$xx'$, we can view it as some inner product in a kernel space induced by
the function $\phi(x) = (1,x,x^2,\ldots)$, where the inner product is diagonal
(but with potentially different weights) in this basis.
For any distinct set of inputs $\{x_1, .., x_n\}$ the set of vectors $\phi(x_i)$ are linearly independent.
The first $n$ columns produce the Vandermonde matrix obtained by stacking rows
$1,x,x,...,x^{n-1}$ for $n$ different values of $x$, which is well known to be non-singular
(since a zero eigenvector would correspond to a degree $n-1$ polynomial with $n$ distinct roots
$\{x_1, .., x_n\}$).

This extends to the case of multidimensional $\x$ if the values, projected
along some dimension, are distinct.
In this case, the kernel space corresponds to the direct sum of copies of $\phi$ applied
elementwise to each coordinate $\x_{i}$. If all the points are distinct and
and far apart from each other, the probability that a given pair coincides under random projection
is negligible. From a union bound, the probability that a given pair coincide is also bounded --
so there must be directions such that projections along that direction are distinct.
Therefore, $\H$ can be considered to be invertible in general.
\end{proof}

As $m \rightarrow \infty$, $\h\h^{\tpose}$ 
concentrates to its expected value. More precisely,
$(\h\h^{\tpose})^{-1}$
approaches $(\H)^{-1}$ for large $m$ if we assume that the smallest 
eigenvalue $\lambda_{min}(\H) \geq \lambda_0$, which
from the above lemma we know
to be true for fixed $n$.
(For the ReLU NTK the difference becomes negligible with high probability for
$m = poly(n/\lambda_0)$ \cite{arora_finegrained_2019}.)
This allows us to replace $\h\h^{\tpose}$ with $\H$ in any bounds involving the former.

We can get learning bounds in terms of $\h\h^{\tpose}$ by studying the upper layer
weights $\w$ of the network after training.
After training, we have $\y = \w\cdot\h$. If $\h\h^{\tpose}$ is invertible
(which the above arguments show is true with high probability
for large $m$),
 the following lemma holds:
\begin{lemma}
\label{lem:upper_layer}
If we initialize a random lower layer and train the weights of the upper layer,
then there exists a solution $\w$
with norm $\sqrt{\y^{\tpose} (\h\h^{\tpose})^{-1} \y}$.
\end{lemma}

\begin{proof}
The minimum norm solution to $\y = \w^{\tpose}\h$ is
\begin{equation}
\w^* = (\h^{\tpose}\h)^{-1}\h^{\tpose}\y.
\end{equation}
The norm squared $(\w^*)^{\tpose}\w^{*}$ of this solution is
given by $\y^{\tpose}\h(\h^{\tpose}\h)^{-2}\h^{\tpose}\y$.

We claim that $\h(\h^{\tpose}\h)^{-2}\h^{\tpose} = (\h\h^{\tpose})^{-1}$.
To show this, consider the SVD decomposition $\h = \m{U}\m{S}\m{V}^{\tpose}$.
Expanding we have
\begin{equation}
\h(\h^{\tpose}\h)^{-2}\h^{\tpose} = \m{U}\m{S}\m{V}^{\tpose} (\m{V}\m{S}^{2} \m{V}^{\tpose})^{-2} \m{V}\m{S}\m{U}^{\tpose}.
\end{equation}
Evaluating the right hand side gets us $\m{U}\m{S}^{-2}\m{U}^{\tpose} = (\h\h^{\tpose})^{-1}$.

Therefore, the norm of the minimum norm solution is $\y^{\tpose} (\h\h^{\tpose})^{-1}\y$.
\end{proof}

We can now complete the proof of Theorem~\ref{thm:kernel_bound}. 
\begin{proof}[Proof of Theorem~\ref{thm:kernel_bound}]
For large $m$, the squared norm
of the weights approaches $\y^{\tpose} (\H)^{-1}\y$.
Since the lower layer is fixed, the optimization problem is linear and therefore convex in the trained weights
$\w$. Therefore
SGD with small learning rate will reach
this optimal solution. The 
Rademacher complexity of this function class is at most $\sqrt{\y^{\tpose}(\H)^{-1}\y}$ which we at most by $\sqrt{M_g}$ where $M_g$ is an upper bound on $\y^{\tpose}(\H)^{-1}\y$.
The optimal solution has $0$ train error based on the 
assumption that $\H$ is full rank and the generalization error will be no more than $O(\sqrt{\frac{\y^{\tpose}(\H)^{-1}\y}{2n}})$ which is at most $\eps$ if we use at least $n = \Omega(M_g/\eps^2)$ training samples - note that this is 
identical to the previous results for training the hidden layer only \cite{arora_finegrained_2019, du_gradient_2019}.
\end{proof}

\subsection{Learning analytic functions}

\label{sec:analytic}

Now, we derive our generalization bounds for single variate functions. We use Theorem \ref{thm:kernel_bound}
to prove the following corollary, a more general version of Corollary 6.2 proven in 
\cite{arora_finegrained_2019} for wide ReLU networks with trainable hidden layer only:

\begin{corollary}
\label{cor:arora}
Consider the function $g:\mathbb{R}^{d}\to\mathbb{R}$ given by:
\begin{equation}
g(\m{x}) = \sum_{k} a_{k}(\bbet_{k}^{\tpose}\x)^{k}
\end{equation}
Then, if $g$ is restricted to $||\m{x}||=1$, and
the NTK or NNGP kernel can be written as $H(\x,\x') = \sum_{k}b_k(\x\cdot\x')^k$,
the function can be learned efficiently with
a wide one-hidden-layer network in the sense of Definition \ref{def:eff_learnable}
with
\begin{equation}
\sqrt{M_{g}} = \sum_{k} b_k^{-1/2} |a_{k}| ||\bbet_{k}||_{2}^{k}
\label{eq:arora_bound}
\end{equation}
up to $g$-independent constants of $O(1)$, where $\bnorm_k\equiv ||\bbet_{k}||_{2}$.
In the particular case of a ReLU network, the bound is
\begin{equation}
\sqrt{M_{g}} = \sum_{k} k |a_{k}| ||\bbet_{k}||_{2}^{k}
\label{eq:arora_bound_relu}
\end{equation}
\end{corollary}
The original corollary applied only to networks with trained hidden layer, and the bound on the ReLu
network excluded odd monomials of power greater than $1$.

\begin{proof}
The extension to NNGP follows from Theorem \ref{thm:kernel_bound}, which allows for the
application of the arguments used to prove
Corollary 6.2 from \cite{arora_finegrained_2019} (particularly those found in Appendix E).

The extension of the ReLu bound to odd powers can be acheived with the following modification.
consider appending a constant
component to the input $\x$
so that the new input to the network is $(\x/\sqrt{2},1/\sqrt{2})$.
The kernel then becomes:
\begin{equation}
\label{eq:ntk-bias}
K(\x,\x') = \frac{\x\cdot\x'+1}{4\pi}\left(\pi-\arccos\left(\frac{\x\cdot\x'+1}{2}\right)\right).
\end{equation}
Re-writing the power series as an expansion around
$\x\cdot\x' = 0$, we have terms of all powers. An asymptotic analysis
of the coefficients using known results
shows that coefficients $b_k$ are asymptotically
$O(k^{-3/2})$ - meaning
in Equation \ref{eq:arora_bound_relu} applies to these
kernels, without restriction to even $k$. 
\end{proof}

Equation \ref{eq:arora_bound} suggests that kernels with slowly decaying (but still convergent)
$b_{k}$ will give the best bounds for learning polynomials.
Many popular kernels do not meet this criteria. For example, for inputs
on the sphere of radius $r$, the Gaussian kernel $K(\x,\x') = e^{-||\x-\x'||^2/2}$
can be written as
$K(\x,\x') = e^{-r^2}e^{\x\cdot\x'}$. This has $b_k^{-1/2} = e^{r^2/2}\sqrt{k!}$, which increases
rapidly with $k$. This provides theoretical
justification for the empirically inferior performance of the Gaussian kernel which we will
present in Section \ref{sec:gravity_experiments}.

Guided by this theory, we focus on
kernels where
$b_{k}^{-1/2} \leq O(k)$, for all $k$ (or, $b_{k}\geq O(k^{-2})$).
The modified ReLu meets this criterion, as well as hand-crafted
kernels of the form
\begin{equation}
K(\x,\x') = \sum_{k} k^{-s}(\x\cdot\x')^k
\end{equation}
with $s\in(1,2]$ is a valid slowly decaying kernel on the sphere. We call these slowly decaying 
kernels.
We note that by Lemma \ref{lem:upper_layer}, the results of Corollary \ref{cor:arora}
apply to networks with output layer training only, as well as kernel learning (which
can be implemented by training wide networks).

Using the extension of Corollary \ref{cor:arora} to odd powers, we first
show that analytic functions with appropriately bounded norms can be learnt. 

\begin{theorem}\label{thm:univar}
Let $g(y)$ be a function analytic around $0$, with radius of convergence
$R_{g}$.
Define the \emph{auxiliary function} $\tilde{g}(y)$ by the power series
\begin{equation}
\tilde{g}(y) = \sum_{k=0}^{\infty} |a_{k}| y^k
\end{equation}
where the $a_k$ are the power series coefficients of $g(y)$. Then the function
$g(\bbet\cdot\x)$, for some fixed vector $\bbet\in\mathbb{R}^{d}$ with $||\x|| = 1$ is efficiently
learnable in the sense of Definition \ref{def:eff_learnable}
using a model with a slowly decaying kernel $K$ with
\begin{equation}
\sqrt{M_{g}} = \bnorm \tilde{g}'(\bnorm)+\tilde{g}(0)
\end{equation}
if the norm $\bnorm\equiv||\bbet||_{2}$ is less than $R_{g}$.
\end{theorem}

\begin{proof}
We first note that the radius of convergence of the power series of $\tilde{g}(y)$ is also
$R_{g}$ since $g(y)$ is analytic. Applying Equation \ref{eq:arora_bound_relu}, pulling out
the $0$th order term, and factoring out $\bnorm$,
we get
\begin{equation}
\sqrt{M_{g}} = |a_{0}|+\bnorm\sum_{k=1}^{\infty}k|a_{k}|\bnorm^{k} = \bnorm\tilde{g}'(\bnorm)+\tilde{g}(0)
\end{equation}
since $\bnorm<R_{g}$.
\end{proof}

The tilde function is the notion of complexity which measures how many samples we need to learn a given function. Informally, the tilde function makes all coefficients in the Taylor series positive. The sample complexity is given by the value of the function at $1$ (in other words, the L1 norm of the coefficients in the Taylor series). For a multivariate function $g(\x)$, we define its tilde function $\tilde{g}(y)$ by substituting any inner product term $\ip{\balpha}{\x}$ by a univariate $y$. The above theorem can then also be generalized to multivariate analytic functions:

\begin{lemma}
\label{lem:multivar}
Given a collection of $p$ vectors $\bbet_{i}$ in $\mathbb{R}^d$,
the function $f(\x) = \prod_{i=1}^{p} \bbet_{i}\cdot \x$ is 
efficiently learnable with
\begin{equation}
\sqrt{M_{f}} = p\prod_{i}\bnorm_{i}
\end{equation}
where $\bnorm_{i}\equiv ||\bbet_{i}||_{2}$.
\end{lemma}

\begin{proof}
The proof of Corollary 6.2 in \cite{arora_finegrained_2019} relied on the following
statement: given positive semi-definite matrices $\m{A}$ and $\m{B}$, with $\m{A} \succeq \m{B}$,
we have:
\begin{equation}
\m{P}_{\m{B}}\m{A}^{-1}\m{P}_{\m{B}} \preceq \m{B}^{+}
\label{eq:psd_ineq}
\end{equation}
where $+$ is the Moore-Penrose pseudoinverse,
and $\m{P}$ is the projection operator.

We can use this result, along with the Taylor expansion of the kernel and a
particular
decomposition of a multivariate monomial in the following way.
Let the matrix $\X$ to be the training data,
such that the $\a$th column $\x_{i}$ is a unit vector in $\mathbb{R}^d$.
Given $\m{K}\equiv \m{X}^{\tpose}\m{X}$, the matrix
of inner products, the Gram matrix $\H$ of the kernel can be written as
\begin{equation}
\H = \sum_{k=0}^{\infty} b_k \m{K}^{\circ k}
\end{equation}
where $\circ$ is the Hadamard (elementwise) product.
Consider the problem of learning the function $f(\x) = \prod_{i=1}^{p} \bbet_{i}\cdot\x$.
Note that we can write:
\begin{equation}
f(\X) = (\X^{\odot k})^{\tpose} \otimes_{i=1}^{k} \bbet_{i}.
\end{equation}

Here $\otimes$ is the tensor product, which for vectors takes an $n_1$-dimensional
vector and an $n_{2}$ dimensional vector as inputs vectors and
returns a $n_{1}n_{2}$ dimensional vector:
\begin{equation}
\m{w}\otimes\m{v} = \begin{pmatrix}
w_1v_1\\
w_{1}v_{2}\\
\cdots\\
w_{1}v_{n_{2}}\\
w_{2}v_{1}\\
\cdots\\
w_{n_{1}}v_{n_{2}}
\end{pmatrix}.
\end{equation}
The operator $\odot$ is the Khatri-Rao product, which takes
an $n_{1}\times n_{3}$ matrix
$\m{A} = (\m{a}_{1},\cdots,\m{a}_{n_{3}})$ and a $n_{2}\otimes n_{3}$ matrix $\m{B} = (\m{b}_{1},\cdots,\m{b}_{n_{3}})$
and returns the $n_{1}n_{2}\times n_{3}$ dimensional matrix
\begin{equation}
\m{A}\odot\m{B} = (\m{a}_{1}\otimes\m{b}_{1},\cdots,\m{a}_{n_{3}}\otimes\m{b}_{n_{3}}).
\end{equation}
For $p=2$, this form of $f(\X)$ can be proved explicitly:
\begin{equation}
(\X^{\odot 2})^{\tpose} \bbet_{1}\otimes\bbet_{2} =
\begin{pmatrix}
\x_{1}\otimes\x_{1},\cdots,\x_{P}\otimes\x_{P}
\end{pmatrix}^{\tpose}\bbet_{1}\otimes\bbet_{2}.
\end{equation}
The $\a$th element of the matrix product is
\begin{equation}
(\x_{\a}\otimes\x_{\a})\cdot(\bbet_{1}\otimes\bbet_{2}) = (\bbet_{1}\cdot\x_{\a})(\bbet_{2}\cdot\x_{\a})
\end{equation}
which is exactly $f(\x_{\a})$.
The formula can be proved for $p>2$ by finite induction.

With this form of $f(\X)$, we can follow the steps of the proof in Appendix E of
\cite{arora_finegrained_2019},
which was written for the case where the $\bbet_{i}$ were identical:
\begin{equation}
\y^{\tpose}(\m{H}^{\infty})^{-1}\y = (\otimes_{i=1}^{p} \bbet_{i})^{\tpose}\X^{\odot p}  (\m{H}^{\infty})^{-1} (\X^{\odot p})^{\tpose} \otimes_{i=1}^{p} \bbet_{i}.
\end{equation}
Using Equation \ref{eq:psd_ineq}, applied to $\m{K}^{\circ p}$, we have:
\begin{equation}
\begin{split}
\y^{\tpose}&(\m{H}^{\infty})^{-1}\y \leq\\ 
 &b_p^{-1}(\otimes_{i=1}^{p} \bbet_{i})^{\tpose}\X^{\odot p} \m{P}_{\m{K}^{\circ p}} (\m{K}^{\circ p})^{+} \m{P}_{\m{K}^{\circ p}} (\X^{\odot p})^{\tpose} \otimes_{i=1}^{p} \bbet_{i}
 \end{split}.
\end{equation}
Since the $\X^{\odot p}$ are eigenvectors of $\m{P}_{\m{K}^{\circ p}}$ with eigenvalue $1$, and
$\X^{\odot p}(\m{K}^{\circ p})^{+}(\X^{\odot p})^{\tpose} = \m{P}_{\X^{\odot p}}$, we have:
\begin{equation}
\y^{\tpose}(\m{H}^{\infty})^{-1}\y \leq 
 b_p^{-1}(\otimes_{i=1}^{p} \bbet_{i})^{\tpose}\m{P}_{\X^{\odot p}} \otimes_{i=1}^{p} \bbet_{i}
\end{equation}
\begin{equation}
\y^{\tpose}(\m{H}^{\infty})^{-1}\y \leq b_p^{-1} \prod_{i=1}^{p} \bbet_{i}\cdot\bbet_{i}.
\end{equation}

For the slowly decaying kernels, $b_p \geq p^{-2}$. Therefore, we have
$\sqrt{\y^{\tpose}(\m{H}^{\infty})^{-1}\y}\leq \sqrt{M_{f}}$ for
\begin{equation}
\sqrt{M_{f}} = p\prod_{i}\bnorm_{i}
\end{equation}
where $\bnorm_{i}\equiv ||\bbet_{i}||_{2}$, as desired.
\end{proof}

This leads to the following generalization of Theorem \ref{thm:univar}:

\begin{theorem}
\label{thm:multivar}

Let $g(\x)$ be a function with multivariate power series representation:
\begin{equation}
g(\x) = \sum_{k} \sum_{v\in V_k} a_{v} \prod_{i=1}^{k} (\bbet_{v,i}\cdot\x)
\end{equation}
where the elements of $V_k$ index the $k$th order
terms of the power series. We define $\tilde{g}(y) = \sum_{k} \tilde{a}_{k} y^k$
with coefficients 
\begin{equation}
\tilde{a}_{k} = \sum_{v\in V_{k}} |a_{v}|\prod_{i=1}^{k}\bnorm_{v,i}.
\end{equation}

If the power series of $\tilde{g}(y)$ converges at $y=1$ then with high probability
$g(\x)$ can be learned efficiently in the sense of Definition \ref{def:eff_learnable}
with $\sqrt{M_{g}} = \tilde{g}'(1)+\tilde{g}(0)$.
\end{theorem}

\begin{proof}
Follow the construction in Theorem \ref{thm:univar}, using Lemma \ref{lem:multivar} to get bounds on the individual terms.
Then sum and evaluate the power series of $\tilde{g}'(1)$ to arrive at the bound.
\end{proof}

\begin{remark}
Note that the $\tilde{g}$ function defined above for multivariate functions depends on the representation, i.e. choice of the vectors $\bbet$. Therefore to be fully formal $\tilde{g}(y)$ should instead be $\tilde{g}_{\bbet}(y)$. For clarity, we drop $\bbet$ from the expression $\tilde{g}_{\bbet}(y)$ and it is implicit in the $\tilde{g}$ notation.
\end{remark}

\begin{remark}\label{rem:learn_app}
If $g(\xb)$ can be approximated by some function $g_{\app}$ such that $|g(\xb)-g_{\app}|\le \eps'$ for all $\xb$ in the unit ball, then Theorem \ref{thm:multivar} can be used to learn $g(\xb)$ within error $\eps'+\eps$ with sample complexity $O(M_{g_{\app}}/\eps^2)$.
\end{remark}

To verify Remark \ref{rem:learn_app}, note that we are doing regression on the upper layer of the neural network, where the lower layer is random. So based on $g_{\app}$ there exists a low-norm solution for the regression coefficients for the upper layer weights which gets error at most $\eps'$. If we solve the regression under the appropriate norm ball, then we get training error at most $\eps'$, and the generalization error will be at most $\eps$ with $O(M_{g_{\app}}/\eps^2)$ samples.

We can also derive the equivalent of the product and chain rule for
function composition.

\begin{proof}[Proof of Corollary~\ref{cor:prod_rule}]
Consider the power series of $g(\x)h(\x)$, which exists and is convergent since each individual
series exists and is convergent. Let the elements of $V_{j,g}$ and $V_{k,h}$ index the $j$th order terms of $g$ and the $k$th order
terms of $h$ respectively.
The individual terms in the series look like:
\begin{equation}
a_{v}b_{w} \prod_{j'=1}^{j} (\bbet_{v,j'}\cdot\x)\prod_{k'=1}^{k} (\bbet_{w,k'}\cdot\x)~\text{for}~v\in V_{j,g},~w\in V_{k,h}
\end{equation}
with bound
\begin{equation}
(j+k)|a_{v}||b_{w}| \prod_{j'=1}^{j} \bnorm_{v,j'}\prod_{k'=1}^{k} \bnorm_{w,k'}~\text{for}~v\in V_{j,g},~w\in V_{k,h}
\end{equation}
for all terms with $j+k >0$ and $\tilde{g}(0)\tilde{h}(0)$ for the term with $j= k = 0$.

Distribute the $j+k$ product, and first focus on the $j$ term only. Summing over all the $V_{k,h}$ for all $k$, we get
\begin{equation}
\begin{split}
\sum_{k} \sum_{w\in V_{k,h}}j |a_{v}||b_{w}| \prod_{j'=1}^{j} \bnorm_{v,j'}\prod_{k'=1}^{k} \bnorm_{w,k'}& = \\ |a_{v}|\prod_{j'=1}^{j} \bnorm_{v,j'}\tilde{h}(1).
\end{split}
\end{equation}
Now summing over the $j$ and $V_{j,g}$ we get $\tilde{g}'(1)\tilde{h}(1)$. If we do the same for the $k$ term,
after summing we get $\tilde{g}(1)\tilde{h}'(1)$. These bounds add and we get the desired formula for
$\sqrt{M_{gh}}$, which, up to the additional $\tilde{g}(0)\tilde{h}(0)$ term looks is the product rule applied to $\tilde{g}$ and $\tilde{h}$.
\end{proof}

One immediate application for this corollary is the product of many
univariate analytic functions. If we define
\begin{equation}
G(\x) = \prod_{i} g_{i}(\bbet_{i}\cdot\x)
\end{equation}
where each of the corresponding $\tilde{g}_{i}(y)$ have the appropriate convergence properties,
then $G$ is efficiently learnable with bound $M_{G}$ given by
\begin{equation}
\sqrt{M_{G}} = \left.\frac{d}{dy} \prod_{i} \tilde{g}_{i}(\bnorm_{i} y)\right|_{y=1}+\prod_{i} \tilde{g}_{i}(0).
\end{equation}

\begin{proof}[Proof of Corollary~\ref{cor:chain_rule}]
Writing out $g(h(\x))$ as a power series in $h(\x)$, we have:
\begin{equation}
g(h(\x)) = \sum_{k=0}^{\infty} a_{k} (h(\x))^{k}.
\end{equation}
We can bound each term individually, and use the $k$-wise product rule to bound each term of
$(h(\x))^{k}$. Doing this, we have:
\begin{equation}
\sqrt{M_{g\circ h}} = \sum_{k=1}^{\infty} k|a_{k}|\tilde{h}'(1)\tilde{h}(1)^{k-1}+\sum_{k=0}^{\infty}|a_{k}|\tilde{h}(0)^{k}.
\end{equation}
Factoring out $\tilde{h}'(1)$ from the first term and then evaluating each of the series gets us the
desired result.
\end{proof}

The following corollary considers the case where the function $g(\x)$ is low-degree and directly follows from Theorem \ref{thm:multivar}.

\begin{fact}\label{fact:tilde_facts} The following facts about the tilde function will be useful in our analysis---
\begin{enumerate}
    \item Given a multivariate analytic function $g(\x)$ of degree $p$ for $\x$ in the
$d$-dimensional unit ball, there is a function
$\tilde{g}(y)$ as defined in Theorem~\ref{thm:multivar} such that $g(\x)$ is learnable to error $\eps$
with $O(p\tilde{g}(1)/\eps^2)$ samples.
\item The tilde of a sum of two functions is at most the sum of the tilde of each of the functions, i.e. if $f=g+ h$ then $\tilde{f}(y) \le \tilde{g}(y) + \tilde{h}(y)$ for $y\ge 0$.
\item 
The tilde of a product of two functions is at most the product of the tilde of each of the functions, i.e. if $f=g\cdot h$ then $\tilde{f}(y) \le \tilde{g}(y) \tilde{h}(y)$ for $y\ge 0$.
\item If $g(\xb)=f(\alpha\xb)$, then $\tilde{g}(y)\le \tilde{f}(\alpha y)$ for $y\ge 0$.
\item If $g(\xb)=f(\xb+\cb)$  for some $\norm{\cb}\le 1$, then $\tilde{g}(y)\le \tilde{f}(y+1)$ for $y\ge 0$. By combining this with the previous fact, if $g(\xb)=f(\alpha(\x - \cb))$ for some $\norm{\cb}\le 1$, then $\tilde{g}(1)\le \tilde{f}(2\alpha)$.
\end{enumerate}

\end{fact}

To verify the last part, note that in the definition of $\tilde{g}$ we replace $\ip{\bbet}{\xb}$ with $y$. Therefore, we will have an additional $\ip{\bbet}{\cb}$ term when we compute the tilde function for $g(\xb)=f(\xb+\cb)$. As $\norm{\cb}\le 1$, the additional term is at most 1.

%The following lemma shows how to get an approximate one-dimensional indicator function with margin $\gamma$ centered at $\alpha$.

The following lemma shows how we can approximate the indicator $\mathbf{1}(x>\alpha)$ with a low-degree polynomial if $x$ is at least $\gamma/2$ far away from $\alpha$. We will use this primitive several times to construct low-degree analytic approximations of indicator functions. The result is based on the following simple fact.

\begin{fact}
 If the Taylor series of $g(\x)$ is exponentially decreasing, then we can truncate it at degree $O(\log(1/\eps))$ to get $\eps$ error. We will use this fact to construct low-degree approximations of functions.
 \end{fact}

\begin{lemma}\label{lem:indicator_poly}
Given a scalar $x$, let the function $$\Phi(x,\gamma,\eps, \alpha)=(1/2)\left(1+\erf\left({(x-\alpha)c\sqrt{\log(1/\eps)}}/{\gamma}\right) \right)$$ for some constant $c$. Let ${\Phi'}(x,\gamma,\eps, \alpha)$ be the function $\Phi(x,\gamma,\eps, \alpha)$ with its Taylor series truncated at degree $O(\log(1/\eps)/\gamma)$. Then for $|\alpha|<1$,
\[ 
{\Phi'}(x,\gamma,\eps, \alpha) =\begin{cases} 
       \eps & x\le \alpha-\gamma/2, \\
      1-\eps & x \ge \alpha+\gamma/2  . 
   \end{cases}
\]
Also, $M_{{\Phi'}}$ is at most $e^{O((\log(1/\eps)/\gamma^2))}$.
\end{lemma}
\begin{proof}
Note that $\Phi(x,\gamma,\eps, \alpha)$ is the cumulative distribution function (cdf) of a normal distribution with mean $\alpha$ and standard deviation $O(\gamma/\sqrt{\log(1/\eps)})$. Note that at most $\eps/100$ of the probability mass of a Gaussian distribution lies more than $O(\sqrt{\log(1/\eps)})$ standard deviations away from the mean. Therefore,
\[ 
\Phi(x,\gamma,\eps, \alpha) =\begin{cases} 
       \eps/100 & x\le \alpha-\gamma/2, \\
      1-\eps/100 & x \ge \alpha+\gamma/2  . 
   \end{cases}
\]
Note that
\begin{align*}
\erf(x)&=\frac{2}{\sqrt{\pi}}\int_0^xe^{-t^2} dt\\
&=\frac{2}{\sqrt{\pi}} \left( \sum_{i=0}^{\infty} \frac{(-1)^{i}x^{2i+1}}{i!(2i+1)}  \right).
\end{align*}
Therefore, the coefficients in the Taylor series expansion of $\erf((x-\alpha)c\sqrt{\log(1/\eps)}/\gamma))$ in terms of $(x-\alpha)$ are smaller than $\eps$ for $i>O(\log(1/\eps)/\gamma^2)$ and are geometrically decreasing henceforth. Therefore, we can truncate the Taylor series at degree $O(\log(1/\eps)/\gamma^2)$ and still have an $O(\eps)$ approximation. Note that for $f(x)=\erf(x)$, 
\begin{align*}
\tilde{f}(y)\le\frac{2}{\sqrt{\pi}} \int_0^y e^{t^2} dt \le \frac{2}{\sqrt{\pi}} ye^{y^2} \le e^{O(y^2)}.
\end{align*}
After shifting by $\alpha$ and scaling by $O(\sqrt{\log(1/\eps)}/\gamma)$, we get $\tilde{\Phi'}(y)=e^{O((y+\alpha)^2\log(1/\eps)/\gamma^2)}$. For $x=1$, this is at most $e^{O(\log(1/\eps)/\gamma^2)}$. Hence the result now follows by Fact \ref{fact:tilde_facts}.

\end{proof}

\subsection{Learnability of cluster based decision node}\label{sec:clusterfunctions}

%Assume that each function $f_j$ is applicable only for a small cluster centered at $\cb_j$ and has radius at most , where $k$ is the number of clusters.

In the informal version of the result for learning cluster based decisions we assumed that the task-codes $\cb$ are prefixed to the input datapoints, which we refer to as $\xb_{\text{inp}}$.  For the formal version of the theorem, we use a small variation. The task code and the input $\cb,\xb_{\text{inp}}$ gets mapped to $\xb=\cb + \xb_{\text{inp}} \cdot (r/3)$ for some constant $r<1/6$. Since $\xb_{\text{inp}}$ resides on the unit sphere, $\xb$ will be distance at most $(r/3)$ from the center it gets mapped to. Note that the overall function $f$ can be written as follows, 
\begin{align*}
    f(\xb) = \sum_{j=1}^k \mathbf{1}\left(\norm{ \xb - \cb_j}^2\le (r/2)^2 \right)  f_j\left((\x - \cb_j)/(r/3) \right)
\end{align*}
where $f_j$ is the function corresponding to the center $\cb_j$. The main idea will be to show that the indicator function can be expressed as an analytic function. 
\begin{theorem}\label{thm:clusterfunctions}
(formal version of Theorem \ref{thm:informa_clusterfunctions}) Assume that $d\ge 10\log k$ (otherwise we can pad by extra coordinates to increase the dimensionality). Then we can find $k$ centers in the unit ball which are at least $r$ apart, for some constant $r$.  Let
$$ 
f(\xb) = \sum_{j=1}^k \mathbf{1}\left(\norm{ \xb - \cb_j}^2\le (r/2)^2 \right)  f_j\left((\x - \cb_j)/(r/3) \right)
$$
where $f_j$ is the function corresponding to the center $\cb_j$. Then if each ${f}_j$ is a degree $p$ polynomial, $M_f$ of the function $f$ is $p\cdot\poly(k/\eps) \sum \tilde{f}_j(6/r)\le p\cdot \poly(k/\eps)  (6/r)^p \sum \tilde{f}_j(1) $.

%If each $f_j$ is a degree $p$ polynomial in variables $<\beta_i, x>$ with $l_1$ norm of the coefficient at most $1$ and $\|\beta_i\|_2 = 1$ for all $i$, then this combined function can be learned with $poly((k/\epsilon)^p)$ examples. 

%More generally the norm is $poly(k/\eps)$ times sum of the norms of $f_j(x - c_j)/(r/4)$. The error in the output at most $\eps$.

\end{theorem}

\begin{proof}
%Let $$ f(\xb) = \sum_{j=1}^{k} p\left( \|\xb - \cb_j \|^2, (r/2)^2, \eps/k, (r/4)^2\right) f_j((\x - \cb_j)/(r/3) ).$$ Note that for data points $\xb$ lying in the $j$-th cluster $\cb_j$, $|f(\cb_j + (r/3) \xb)-f_j(\xb)| \le \eps$. 
Let
\begin{align*}
    f_{\app}(\xb) = \sum_{j=1}^k{\Phi'}\left( \|\xb - \cb_j \|^2, (r/2)^2, \eps/k, (r/4)^2\right) f_j\left((\x - \cb_j)/(r/3) \right)
\end{align*}
where ${\Phi'}$ is defined in Lemma \ref{lem:indicator_poly}. Let 
\[
I_j(\xb)={\Phi'}( \|\xb - \cb_j \|^2, (r/2)^2, \eps/k, (r/4)^2).
\]
The indicator $I_j(\xb)$ checks if $\|\xb - \cb_j \|$ is a constant fraction less than $r/2$, or a constant fraction more than $r/2$. Note that if $\x$ is from a different cluster, then $\|\x - \cb_j\|$ is at least some constant, and hence $I_j(\x)$ is at most $\epsilon/k$. The contribution from $k$ such clusters would be at most $\epsilon$. If $\| \x - \cb_j \| < \epsilon/k$, then the indicator is at least $1- O(\epsilon/k)$. Hence as $f_{\app}$ is an $O(\eps)$-approximation to $f$, by Remark \ref{rem:learn_app} it suffices to show learnability of $f_{\app}$. 

If $y = \langle \xb, \cb_j \rangle$ and assuming $\xb$ and the centers $\cb_j$ are all on unit sphere,
\[
\tilde{I}_{j}(y) = \tilde{\Phi'}(2+2y, r/3, \eps/k, r/3)\le e^{O(\log(k/\eps)}= \poly(k/\epsilon).
\]

By Fact \ref{fact:tilde_facts},
\begin{align*}
    \tilde{f}(y) \le \poly(k/\epsilon) \sum_j \tilde{f}_j(6/r).
\end{align*}
As $f_j$ are at most degree $p$,
\begin{align*}
    \tilde{f}(y) \le \poly(k/\epsilon) \sum_j \tilde{f}_j(6/r)\le p \cdot \poly(k/\eps) (6/r)^p \sum \tilde{f}_j(1) .
\end{align*}

\end{proof}

\begin{corollary}\label{cor:lookup}
The previous theorem implies that we can also learn $f$ where $f$ is a lookup table with $M_f=\poly(k/\eps)$, as long as the keys $c_i$ are well separated. Note that as long as the keys $c_i$ are distinct (for example, names) we can hash them to random vectors on a sphere so that they are all well-separated.
%[[Compare this with other method where you use $log k$ prefix bits. There you get dependence $d^{\log k}$. Here you don't get that but instead you get increased norm due to scaling input of $f$ by $1/\alpha$]]
\end{corollary}

Note that the indicator function for the informal version of Theorem \ref{thm:clusterfunctions} stated in the main body is the same as that for the lookup table in Corollary \ref{cor:lookup}. Therefore, the informal version of Theorem \ref{thm:clusterfunctions} follows as a Corollary of Theorem \ref{thm:clusterfunctions}.

\subsection{Learnability of functions defined on leaves of a decision tree}\label{sec:decision_app}

We consider decision trees on inputs drawn from $\{-1,1\}^d$. We show that such a decision tree $g$ can be learnt with $M_g\le O(d^h)$. From this section onwards, we view the combined input $\cb,\xb$ as $\xb$. 

The decision tree $g$ can be written as follows,
\[
g(\mathbf{x}) = \sum_{j} I_j(\mathbf{x}) v_j,
\]
where the summation runs over all the leaves, $I_j(\mathbf{x})$ is the indicator function for leaf $j$, and $v_j\in [-1,1]$ is the constant value on the leaf $j$. We scale the inputs by $\sqrt{d}$ to make them lie on the unit sphere, and hence each coordinate of $\xb$ is either $\pm 1/\sqrt{d}$.

Let the total number of leaves in the decision tree be $B$. The decision tree indicator function of the $j$-th leaf can be written as the product over the path of all internal decision nodes. Let $j_l$ be  variable at the $l$-th decision node on the path used by the $j$-th leaf. We can write,
\[
I_{j}(\mathbf{x}) = \prod_l \left(a_{j_l} x_{j_l} + b_{j_l}\right),
\]
where each $x_{j_l} \in \{-1/\sqrt{d}, 1/\sqrt{d}\}$  and $a_{j_l} \in \{-\sqrt{d}/2, \sqrt{d}/2\}$ and $b_{j_l} \in \{-1/2, 1/2\}$. Note that the values of $a_{j_l}$ and $b_{j,l}$ are chosen depending on whether the path for the $j$-th leaf choses the left child or the right child at the $l$-th decision variable. For ease of exposition, the following theorem is stated for the case where the leaf functions are constant functions, and the case where there are some analytic functions at the leaves also follows in the same way.

\begin{theorem}\label{thm:decisiontree}
If a function is given by $g(\mathbf{x}) = \sum_{j=1}^B I_j(\mathbf{x}) v_j$, where $I_j(\mathbf{x})$ is a leaf indicator function in the above form, with tree depth $h$, then $M_g$ is at most $O(d^{h})$. %[[This is essentially encoding different functions using $log k$ prefix bits. ]]
\end{theorem}

\begin{proof}
Note that 
\begin{align*}
    \tilde{g}(y)&\le\sum \tilde{I}_j(y)|v_j|\\
    &\le \sum\prod_l \left(\sqrt{d} y/2 + 1/2\right)\\
    \implies 
    \tilde{g}(1) 
    &\le 2^h(\sqrt{d}/2+1/2)^h\le d^{h}.
\end{align*}
As the degree of $g$ is at most $h$, therefore $M_g \le h\tilde{g}(1) \le hd^h$.
\end{proof}

\begin{remark}
Note that by Theorem \ref{thm:decisiontree} we need $O\left((\log k)^{\log k}\eps^{-2}\right)$ samples to learn a lookup table based on a decision tree. On the other hand, by Corollary \ref{cor:lookup} we need $\poly(k/\eps)$ samples to learn a lookup table using cluster based decision nodes. This shows that using a hash function to obtain a  random $O(\log k)$ bit encoding of the indexes for the $k$ lookups is more efficient than using a fixed $\log k$ length encoding for the $k$ lookups.
\end{remark}

We also prove a corresponding lower bound in Theorem \ref{thm:parity} which shows that $d^{\Omega(h)}$ samples are necessary to learn decision trees of depth $h$.

We will now consider decision trees where the branching is based on the inner product of $\xb$ with some direction $\bbet_{j,l}$. Assume that there is a constant gap for each decision split, then the decision tree indicator function can be written as,
\[
I_{j}(\mathbf{x}) = \prod_l \mathbf{1}(\langle \xb,\bbet_{j,l}\rangle>\alpha_{j,l}).
\]
%If the decision tree indicator function of the $j$-th leaf is defined in terms of real variables $\mathbf{x} \in \mathbb{S}^{d-1} \subset \mathbb{R}^d$ (note we assume that $x_{d} = 1/\sqrt{d}$ so that $\mathbf{x}$ contains a constant coordinate). Assume that there is a constant gap for each decision split, then the decision tree indicator function can be written as
%\[
%I_{j}(\mathbf{x}) = \prod_l \left(a_{j_l} \mathrm{erf} \left( \frac{x_{j_l} - \alpha_{j_l}/\sqrt{d}}{\gamma} \right) + b_{j_l}\right),
%\]

%where $\mathrm{erf}$ is the error function, $\gamma$ is a constant specifying the split gap, $\alpha_{j_l}/\sqrt{d}$ is the split threshold for variable $x_{j_l}$ with $\alpha_{j_l} \leq E$,  $a_{j_l} \in \{-1/2, 1/2\}$ and $b_{j_l} \in \{-1/2, 1/2\}$.

%\emph{Note that this is over unit sphere, previous was over unit cube}

\begin{theorem}\label{thm:decisiontreemargin}
(formal version of Theorem \ref{thm:decisiontreemargin_informal})
A decision tree of depth $h$ where every node partitions in a certain direction with margin $\gamma$ can be written as $g(\mathbf{x}) = \sum_{j=1}^B I_j(\mathbf{x}) f_j(\mathbf{x})$, then the final
$$M_g=e^{O(h\log(1/\eps)/\gamma^2)} (p+h\log1/\eps)\sum  \tilde{f}_j(1),$$
where $p$ is the maximum degree of $f_j$.
\end{theorem}
\begin{proof}
Define $g_{\app}$,
$$
{g}_{\app}(\mathbf{x}) = \sum_{j=1}^B \Pi_l \Phi'(\langle \xb,\bbet_{j,l}\rangle,\gamma,\eps/h,\alpha_{j,l} )f_j(\mathbf{x})
$$
where $\Phi'$ is as defined in Lemma \ref{lem:indicator_poly}. Note that for all $y= 1$,
$$
\tilde{\Phi'}(1,\gamma,\eps/h,\alpha_{j,l}) \le e^{O(\log(1/\eps)/\gamma^2)}.
$$
Therefore,
\begin{align*}
\tilde{g}_{\app}(1) &\le \sum_{j=1}^B \Pi_l \tilde{\Phi'}(1,\gamma,\eps/h,\alpha_{j,l} )\tilde{f}_j(1),\\
&\le  e^{O(\log(1/\eps)/\gamma^2)}\sum \tilde{f}_j(1).
\end{align*}
Note that the degree of $g_{\app}$ is at most $O(p+h\log(1/\eps)/\gamma^2)$. Therefore, 
$$
M_{g_{\app}} \le e^{O(h\log(1/\eps)/\gamma^2)} (p+h\log(1/\eps)/\gamma^2)\sum  \tilde{f}_j(1).
$$
By Remark \ref{rem:learn_app}, learnability of $g$ follows from the learnability of its analytic approximation $g_{\app}$.
\end{proof}

\subsection{Generalized Decision Program}\label{sec:program_app}

In this section, instead a decision tree, we will consider a circuit with fan-out 1, where each gate (node) evaluates some function of the values returned by its children and the input $\x$. A decision tree is a special case of such circuits in which the gates are all switches. 

So far, the function outputs were univariate but we will now generalize and allow multivariate (vector) outputs as well. Hence the functions can now evaluate and return data structures, represented by vectors. We assume that each output is at most $d$ dimensional and lies in the unit ball. 

\begin{definition}\label{def:multivar_tilde}
For a multivariate output function $f$, we define $\tilde{f}(y)$ as the sum of $\tilde{f}_i(y)$ for each of the output coordinates $f_i$.
\end{definition}

\begin{remark}
Theorem \ref{thm:clusterfunctions} , \ref{thm:decisiontree} and \ref{thm:decisiontreemargin} extend to the multivariate output case. Note that if each of the individual functions has degree at most $p$, then the sample complexity for learning the multivariate output $f$ is at most $O(p\tilde{f}(1)/\eps^2))$ (where the multivariate tilde function is defined in Definition \ref{def:multivar_tilde}).
\end{remark}

% Our previous results directly extend to the multivariate case. For instance, Theorem \ref{thm:clusterfunctions} still applies except that we will use $\tilde{f}$ for the vector function for each cluster. Similarly for Theorems \ref{thm:decisiontree} and \ref{thm:decisiontreemargin}.

%1. Object representation: combining fields into a single object, and accessing a certain field given an object. By packing the different fields in an object into orthogonal subspaces, we can get a sketched representation of the object. If the different fields are computed by functions representable by subtrees in the class C, then the entire The fields of the objects can be accessed by projection. 

We now define a  generalized decision program and the class of functions that we support.

\begin{definition}
We define a generalized decision program to be a circuit with fan-out 1 (i.e., a tree topology) where each gate evaluates a function of the values returned by its children and the input $\x$, and the root node evaluates the final output. All gates, including those at the leaves, have access to the input $\x$. We support the following gate operations. Let $h$ be the output of a gate, let each gate have at most $k$ children, and let $\{f_1,\dots,f_k\}$ be the outputs of its children.

\begin{enumerate}
    \item Any analytic function of the child gates of degree at most $p$, including sum $h=\sum_{i=1}^{k} f_i$ and product of $p$ terms $h=\Pi_{i=1}^p f_i$.
    \item Margin based switch (decision) gate with children $\{f_1,f_2\}$, some constant margin $\gamma$, vector $\bbet$ and constant $\alpha$,
    \[h= \begin{cases} 
      f_1 & \text{ if }\ip{\bbet}{\x}-\alpha\leq -\gamma/2, \\
      f_2 & \text{ if } \ip{\bbet}{\x}-\alpha\geq \gamma/2.
   \end{cases}
\]
    \item Cluster based switch gate with $k$ centers $\{\code^{(1)},\dots,\code^{(k)}\}$, with separation $r$ for some constant $r$, and the output is $f_i$ if $\norm{\x-\code^{(i)}}\le r/3$. A special case of this is a look-up table which returns value $v_i$ if $\x=\code^{(i)}$, and 0 if $\x$ does not match any of the centers.
    \item Create a data structure out of separate fields by concatenation such as constructing a tuple $[f_1,\dots, f_k]$ which creates a single data structure out of its children, or extract a field out of a data structure.   
    \item Given a table $T$ with $k$ entries $\{r_1,\dots, r_k\}$, a Boolean-valued function $p$ and an analytic function $f$,  SQL queries of the form \verb|SELECT SUM f(r_i), WHERE p(r_i, x)|. Here, we assume that $f$ has bounded value and $p$ can be approximated by an analytic function of degree at most $p$.
    \item Compositions of functions,  $h(\x)=f(g(\x))$.
    \end{enumerate}

\end{definition}

First, we note that all of the above operators can be approximated by low-degree polynomials.

\begin{claim}
If $p\le O(\log (k/\eps))$, each of the above operators in the generalized decision program can be expressed as a polynomial of degree at most $O(\log(k/\eps))$, where $k$ is maximum out-degree of any of the nodes.

% specify approximation parameters in terms of complexity of poly

% compose poly to get results

\end{claim}

\begin{remark}
Note that for the SQL query, we can also approximate other aggregation operators apart from SUM, such as MAX or MIN. For example, to approximate MAX of $x_1,\dots,x_k$ up to $\eps$ where the input lies between $[0,1]$ we can first write it as 
\begin{align*}
   \text{MAX}(x_1,\dots,x_k) =\eps \sum_j \mathbf{1}\left( \sum_i (\mathbf{1}(x_i>\eps j)>1/2)\right),
\end{align*}
and then approximate the indicators by analytic functions.
\end{remark}

Lemma \ref{cor:tilde} shows how we can compute the tilde function of the generalized decision program.

\begin{lemma}\label{cor:tilde}
The tilde function for a generalized decision program can be computed recursively with the following steps:

\begin{enumerate}
    \item For a sum gate $h=f+g$, $\tilde{h}(y)=\tilde{f}(y)+\tilde{g}(y)$.
    \item For a product gate, $h=f.g$, $\tilde{h}(y)=\tilde{f}(y)\cdot\tilde{g}(y)$.
    \item For a margin based decision gate (switch) with children $f$ and $g$, $h=I_{left}f+(1-I_{left})g$ and $\tilde{h}(y)=\tilde{I}_{left}(\tilde{f}(y)+\tilde{g}(y))+\tilde{g}(y)$. Here $I_{left}$ is the indicator for the case where the left child is chosen.
    \item For cluster based decision gate (switch) with children $\{f_1,...,f_k\}$, $\tilde{h}(y)\le \sum_i \tilde{I_i}\tilde{f}_i(6y/r)$. Here $I_i$ is the indicator for the cluster corresponding to the $i$-th child.  
    \item For a look-up table with $k$ key-values, $\tilde{h}(y)\le k\tilde{I}(y)$ as long as the $\ell_1$ norm of each key-value is at most 1.  
    \item Creating a data structure out of separate fields can be done by concatenation, and $\tilde{h}$ for the result is at most sum of the original tilde functions. Extracting a field out of a data structure can also be done in the same way.  
    \item Given an analytic function $f$ and a Boolean function $p$, for a SQL operator $h$ over a table $T$ with $k$ entries $\{r_1,\dots, r_k\}$ representing \verb|SELECT SUM f(r_i), WHERE p(r_i, x)|, or in other words $h=\sum_i f(r_i) p(r_i,x)$, $\tilde{h}(y)\le\sum_i \tilde{I}_{p,r_i}(y)$, where ${I}_{p,r_i}$ is the indicator for $p(r_i, x)$. For example, $x$ here can denote some threshold value to be applied to a column of the table, or selecting some subset of entries (in Fig. \ref{fig:learn}, $x$ is the zip-code).
    \item For $h(\x)=f(g(\x))$, $\tilde{h}(y)\le \tilde{f}(\tilde{g}(y))$.
\end{enumerate}
\end{lemma}

All except for the last part of the above Lemma directly follow from the results in the previous sub-section. Below, we prove the result for the last part regarding function compositions.

\begin{lemma}
Assume that all functions have input and output dimension at most $d$. If $f$ and $g$ are two functions with degree at most $p_1$ and $p_2$, then $h(\x)=f(g(\x))$ has degree at most $p_1p_2$ and $\tilde{h}(y)\le \tilde{f}(\tilde{g}(y))$.
\end{lemma}
\begin{proof}
Note that this follows if $f$ and $g$ are both scalar outputs and inputs. Let $g(\x)=(g_1(\x), ..., g_d(\x))$. Let us begin with the case where $f=\ip{\bbet}{\x}$, where $\|\bbet\|=1$. Then $\tilde{h}(y)=\sum_i |\beta_i| \tilde{g}_i(y)\le \sum_i \tilde{g}_i(y)\le \tilde{g}(y)$. When $f=\Pi_{i=1}^{p_1}\ip{\bbet_i}{\x}$, $\tilde{h}(y)\le \tilde{g}(y)^{p_1} \le \tilde{f}(\tilde{g}(y))$. The same argument works when we take a linear combination, and also for a multivariate function $f$ (as $\tilde{f}$ for a multivariate $f$ is the summation of individual $\tilde{f}_i$, by definition).
\end{proof}

We now present our result for learning generalized decision programs.

\begin{theorem}\label{thm:general_learning} Let the in-degree of any gate be at most $k$. The sample complexity for learning the following classes of generalized decision programs is as follows:

\begin{enumerate}
    \item If every gate is either a decision node with margin $\gamma$, a sum gate, or a lookup of size at most $k$, then $M_g\le e^{O(h\log(1/\eps)/\gamma^2)}k^{O(h)}$.
    \item For some constant $C$, if there are at most $C$ product gates with degree at most $C$, and every other gate is a decision gate with margin $\gamma$ or a sum gate with constant functions at the leaves, then $M_g\le e^{O(h\log(1/\eps)/\gamma^2)}$.
    \item Given a function $f$ and a Boolean function $p$ which can be approximated by a polynomial of degree at most $O(\log (k/\eps))$, for a SQL operator $g$ over a table $T$ with $k$ entries $\{r_1,\dots, r_k\}$ representing \verb|SELECT SUM f(r_i), WHERE p(r_i, x)|, $M_g\le \sum_i \tilde{I}_{p,r_i} (1)$.
    \item Let the function at every gate be an analytic function $f$ of degree at most $p$ and the sum of the coefficients of $f$ is upper bounded by $c^p$ for some constant $c$. Then note that $\tilde{f}(y)\le (cy)^p$ for $y\ge 1$. Therefore, the final function $\tilde{g}(y) \le (cky)^{p^h}$ and hence $M_g\le (ck)^{p^h}$.
\end{enumerate}
\end{theorem}
\begin{proof}
The first three claims can be obtained using Lemma \ref{cor:tilde}.

For the final claim, consider the final polynomial obtained by expanding the function at each gate in a bottom-up way. We will upper bound $\tilde{g}(y)$ for the overall function $g$ corresponding to the generalized decision program. $\tilde{g}(y)$ can be upper bounded by starting with $\tilde{f}(y)$ for the leaf nodes $f$. For any internal gate $i$, let $g_i(x)= f_i(f_{j_1}(x),\dots, f_{j_p}(x))$ where $f_{j_t}$ are the outputs of the children of the gate $i$. We recursively compute $\tilde{g}_i(y)=\tilde{f}_i(\sum_{l} \tilde{f}_{j_l}(y))$. Therefore, for a gate with $k$ children $\tilde{g}_{i}(y)\le (c \sum_{l}\tilde{g}_{j_l}(y))^p$. Therefore, for the root gate $g_0$, $\tilde{g}_0(y)\le (cky)^{p^h}$.
\end{proof}

\begin{remark}
Note that the dependence on $h$ is doubly exponential. We show a corresponding lower bound in Theorem \ref{thm:depth_general} that this is necessary.
\end{remark}

Theorem \ref{thm:general_learning} implies that we can learn programs such as the following formal version of Fig. \ref{fig:learn} (right)---which involves analytic functions, SQL queries, data structures, and table look-up. 

\begin{example}
Consider the following program:
\begin{verbatim}
class Person{
    string name;
    Address address;
    int income;
    public string get_zip_code(){
        return address.zip_code;
    }
    init(input_name, input_address, input_income){
        name = input_name;
        address = input_address;
        income = input_income;
    }
}
class Address{
    int street_number;
    string street_name;
    string city;
    string state;
    string zip_code;
    public string get_zip_code(){
        return zip_code;
    }
    init(...){
        ... # function to create new object with input values
    }
}
dictionary name_to_address_table;
dictionary zip_code_to_lat_long; #maps zip_code to tuple of (latitute, longitude)

boolean in_same_zip_code(Person A, Person B){
    return A.get_zip_code() == B.get_zip_code();
}

float get_straight_line_distance(Person A, Person B){
    lat_longA =  zip_code_to_lat_long[A.get_zip_code()];
    lat_longB =  zip_code_to_lat_long[B.get_zip_code()];
    return euclidean_distance(lat_longA, lat_longB);
}

float avg_income_zip_code(string zip_code){
    construct SQL table T with income, zip_code from name_to_address_table;
    return output of SQL query "SELECT AVG(INCOME) FROM T WHERE ZIP_CODE=zip_code"
}

\end{verbatim}
\end{example}

The following claim follows from Theorem \ref{thm:general_learning}.

\begin{claim}
The above classes and functions can be implemented and learnt using $(k/\eps)^{O(\log(1/\eps))}$ samples, where the tables are of size at most $k$.
\end{claim} 
\begin{proof}

We begin with the \verb|in_same_zip_code()| function. Note that this is a special case of the cluster based functions. As in Corollary \ref{cor:lookup} all attributes such as zip-code are appropriately hashed such that they are well-separated. We can now test equality by doing an indicator function for a ball around the zip-code of Person A. The indicator function for a ball can be approximated by a low-degree polynomial as in the cluster-based branching results in Theorem \ref{thm:clusterfunctions}. As the total number of individuals is at most $k$, therefore by Theorem \ref{thm:clusterfunctions} the sample complexity is at most  $\poly(k/\eps)$. 

For the  \verb|avg_income_zip_code()| function, we use the SQL query result in Theorem \ref{thm:general_learning}. Note that the indicators are testing equality in the case of our program, and hence as in the previous case we can use the cluster-based branching  result in Theorem \ref{thm:clusterfunctions} to approximate these indicators by polynomial functions, to obtain a sample complexity of $\poly(k/\eps)$. 

Finally, we argue that we can learn the \verb|get_straight_line_distance()| function. Here, we are composing two functions $f$ and $(g_1,g_2)$ where $f$ is the distance function and $(g_1,g_2)$ are the lookups for the latitude and longitude for Person A and B. By Corollary \ref{cor:lookup}, the lookups have $\tilde{g_i}(1)\le \poly(k/\eps)$. By part 6 of Lemma \ref{cor:tilde}, the tilde for the concatenation is the sum of the tilde for the individual functions.  For computing the Euclidean distance $\sqrt{\sum(x_i-y_i)^2}$, note that the square root function does not have a Taylor series defined at 0. However, we can use the same analysis as in the proof for learning the $1/x$ function in the gravitational law (see Appendix \ref{sec:learning_gravity}) to get a polynomial of degree at most  $O(\log(1/\eps))$, and hence $\tilde{f}(y)\le (O(y))^{\log(1/\eps)}$. Thus using the composition rule in Lemma \ref{cor:tilde}, the sample complexity is $(k/\eps)^{O(\log(1/\eps))}$.

\end{proof}

\section{Learning dynamical systems}

\subsection{Gravitational force law}

\label{sec:learning_gravity}

We can use the product and chain rules to show that many functions
important in scientific applications can be efficiently learnable.
This is true even when the function has a singularity. As an example
demonstrating both, we prove the following bound on learning Newton's law
of gravitation:

\begin{theorem}
\label{thm:gravity_bound}
Consider a system of $k$ bodies with positions $\x_{i}\in \mathbb{R}^{3}$
and masses $m_{i}$, interacting via the force:
\begin{equation}
\F_{i} = \sum_{j\neq i} \frac{m_{i}m_{j}}{r_{ij}^3}(\x_{j}-\x_{i})
\end{equation}
where $r_{ij} \equiv ||\x_{i}-\x_{j}||$.
We assume that $\rrat = r_{max}/r_{min}$, the ratio between the largest
and smallest pairwise distance between any two bodies, is constant. 
Suppose the $m_i$ have been rescaled
to be between $0$ and $1$.
Then the force law is efficiently
learnable in the sense of Definition \ref{def:eff_learnable} using the modified ReLU kernel
to generalization error less than $\eps$
using $k^{O(\ln(k/\eps))}$ samples.
\end{theorem}

\begin{proof}
We will prove learning bounds for each component of $F$ separately,
showing efficient learning with probability greater than $1-\delta/3k$. Then, using the union bound,
the probability of simultaneously learning all the components efficiently will be $1-\delta$.

There are two levels of approximation: first, we will construct a function which is within $\epsilon/2$
of the original force law, but more learnable. Secondly, we will prove bounds on learning that function
to within error $\epsilon/2$.

We first rescale the vector of collective $\{\x_{i}\}$ so that their collective length is at most $1$.
In these new units, this gives us $r_{max}^2\leq \frac{2}{k}$.
The first component
of the force on $\x_{1}$ can be written as:
\begin{equation}
(\F_{1})_{1} = \sum_{j=2}^{k} \frac{m_{1}m_{j}}{r_{1j}^{2}}\frac{((\x_{j})_{1}-(\x_{1})_{1})}{r_{1j}}.
\end{equation} 
If we find a bound $\sqrt{M_{f}}$ for an individual contribution $f$
to the force, we can get a bound on the total
$\sqrt{M_{F}} = (k-1)\sqrt{M_{f}}$. Consider an individual force term
in the sum. The force has a singularity at
$r_{1j} = 0$. In addition, the function $r_{1j}$ itself is non-analytic due to the branch cut at $0$.

We instead will approximate the force law with a finite power series in $r_{1j}^2$, and get bounds on learning said
power series. The power series representation of
$(1-x)^{-3/2}$ is $\sum_{n=0}^{\infty} \frac{(2n+1)!!}{(2n)!!}x^{n}$.
If we approximate the function with $d$ terms, the error can be 
bounded using Taylor's theorem. The Lagrange form
of the error gives us the bound
\begin{equation}
\left|\frac{1}{(1-x)^{3/2}}-\sum_{n=0}^{d} \frac{(2n+1)!!}{ (2n)!!}x^{n}\right| \leq \frac{\sqrt{\pi d} |x|^{d+1}}{(1-|x|)^{5/2+d}}
\end{equation}
where we use $\frac{(2n+1)!!}{(2n)!!}\approx \sqrt{\pi n}$
for large $n$.
We can use the above expansion by
rewriting
\begin{equation}
r_{1j}^{-3} = a^{-3}(1-(1-r_{1j}^{2}/a^{2}))^{-3/2}
\end{equation}
for some shift $a$.
Approximation with $f_{d}(r_{1j}^2)$, the first $d$ terms of the power series in $(1-r_{1j}^{2}/a^{2})$ gives us the error:
\begin{equation}
|f_{d}(r_{1j}^2)-r_{1j}^{-3}|\leq \frac{\sqrt{\pi d}|1-r_{1j}^{2}/a^{2}|^{d+1}}{a^3(1-|1-r_{1j}^{2}/a^{2}|)^{5/2+d}}
\end{equation}
which we want to be small over the range $r_{min}\leq r_{1j}\leq r_{max}$.

The bound is optimized when it takes the same value at
$r_{min}$ and $r_{max}$, so we set
$a^{2} = (r_{min}^2+r_{max}^2)/2$.
In the limit that $r_{max}\gg r_{min}$, where learning is most difficult,
the bound becomes
\begin{equation}
|f_{d}(r_{1j}^2)-r_{1j}^{-3}|\leq 
\frac{\sqrt{8\pi d}}{r_{max}^3}
\left(\rrat^2/2\right)^{5/2+d}e^{-2(d+1)/\rrat^2}
\end{equation}
where $\rrat = r_{max}/r_{min}$, which is constant
by assumption.

In order to estimate an individual contribution to the force force to error $\epsilon/2k$ (so the total error
is $\epsilon/2$), we must have:
\begin{equation}
m_1 m_j r_{max}  |f_{d}(r_{1j})-r_{1j}^{-3}|\leq \frac{\eps}{2k}
\end{equation}
This allows us to choose the smallest $d$ which gives us this error. Taking
the logarithm of both sides, we have:
\begin{equation}
\frac{1}{2}\ln(d)-(5/2+d)\ln\left(2/\rrat^2\right)-2(d+1)/\rrat^2\leq \ln(\eps/k^2).
\end{equation}
where we use that
$r_{max}^2\leq 2/k$ after rescaling.
The choice $d\geq \rrat^2\ln(k^2/\eps)$ ensures error less than $\eps/2k$ per term.

Using this approximation, we can use the product and chain rules to 
get learning bounds on
the force law. We can write the approximation
\begin{equation}
F_{\eps}(\x) = \sum_{j\neq 1} m_1m_j f_{d}(h_{j}(\x)) k_{j}(\x)
\end{equation}
where $h_{j}(\x) = ||\x_{1}-\x_{j}||$
and $k_{j}(\x) = (\x_{1})_1-(\x_{j})_{j}$
The number of samples needed for efficient learning is bounded by
$\sqrt{M_{F_{\eps}}} = \frac{\sqrt{8}k}{r_{max}^3} A_{F_{\eps}}$, for
\begin{equation}
A_{F_{\eps}} =
\tilde{f}_{d}'(\tilde{h}(1))\tilde{h}'(1) \tilde{k}(1) +\tilde{f}_{d}(\tilde{h}(1)) \tilde{k}'(1)
\end{equation}
with
\begin{equation}
\tilde{k}(y) = \sqrt{2}y,~\tilde{h}(y) = 6y^2,~\tilde{f}_{d}(y) = \sqrt{\pi d}(1+y/a^2)^{d}.
\end{equation}
Evaluating, we have
\begin{equation}
A_{F_{\eps}}   = 
 \sqrt{2\pi d}\left(1+\frac{12}{r_{max}^2}\right)^{d} +\sqrt{\pi d^3}\left(1+\frac{12}{r_{max}^2}\right)^{d-1}
\end{equation}
which, after using $r_{max}^2\leq 2/k$ and
$d = \rrat^2\ln(k^2/\eps)$
gives us the bound
\begin{equation}
\sqrt{M_{F_{\eps}}} \leq  k^{-1/2}  \left(\rrat^2\ln(k^2/\eps)\right)^{3/2}\left(24 k\right)^{R^2\ln(k^2/\eps)}.
\label{eq:gravity_bound}
\end{equation}
The asymptotic behavior is
\begin{equation}
\sqrt{M_{F_{\eps}}} = k^{O(\ln(k/\eps))}
\end{equation}
since $\rrat$ is bounded.

We can therefore
learn an $\epsilon/2$-approximation of one component of $\F_{1}$, with probability
at least $1-\delta/3k$ and error $\epsilon/2$ with  $O(4(M_{F_{\eps}}+\log(3k/\delta))/\epsilon^2)$
samples. Therefore, we can learn $\F_{1}$ to error $\epsilon$ with the same number
of samples. Using a union bound, with probability at least $1-\delta$ we can simultaneously
learn all components of all $\{\F_{i}\}$ with that number of samples.
\end{proof}

We note that since the cutoff of the power series at $d(\eps) = O(\rrat^2\ln(k^2/\eps))$ dominates
the bound, we can easily compute learning bounds for other power-series kernels as well.
If the $d$th power series coefficient of the kernel is $b_d$, then 
the bound on $\sqrt{M_{F_{\eps}}}$
is increased by $(d(\eps)^2b_{d(\eps)})^{-1/2}$.
For example, for the Gaussian kernel, since $b_d^{-1/2} = \sqrt{d!}$, the bound becomes
\begin{equation}
\sqrt{M_{F_{\eps}}} = (\rrat^2\ln(k^2/\eps)k)^{O(\ln(k/\eps))}
\end{equation}
which increases the exponent of $k$ by a
factor of $\ln(\rrat^2\ln(k^2/\eps))$.

\subsection{Empirical confirmation of learning bounds}

\label{sec:gravity_experiments}

We empirically validated our analytical learning bounds by
training models to
learn the gravitational force function for $k$ bodies (with $k$ 
ranging from $5$ to $400$) in a $3-$dimensional space.
We created synthetic datasets by randomly drawing $k$ points from
$[0, 1]^3$ corresponding to the location of $k$ bodies, and compute 
the gravitational force (according to Figure~\ref{fig:learn}) on a
target body also drawn randomly from $[0, 1]^3$. To avoid 
singularities, we ensured a minimum distance of $0.1$ between the 
target body and the other bodies
(corresponding to the choice $\rrat = 10$). 
As predicted by the theory, none of the models learn well
if $\rrat$ is not fixed.
We randomly drew the 
masses 
corresponding to the $k+1$ bodies from $[0,10]$. We generated $5$ 
million such examples - each example with $4(k+1)$ features 
corresponding to the location and mass of each of the bodies, and a 
single label corresponding to the gravitational force $F$ on the target 
body along the $x$-axis. We held out 
$10\%$ of the dataset as test data to compute the root mean square 
error (RMSE) in prediction.
We trained three different neural networks 
on this data, corresponding to various kernels we analyzed in the 
previous section:
\begin{enumerate}
\item A wide one hidden-layer ReLU network (corresponding
to the ReLU NTK kernel).
\item A wide one hidden-layer ReLU network with a 
constant bias feature added to the input (corresponding to the NTK 
kernel).
\item A wide one hidden-layer network with exponential 
activation function, where only the top layer of the network is
trained (corresponding to the Gaussian kernel).
\end{enumerate}

We used a hidden layer of width $1000$ for all the
networks, as we observed that 
increasing the network width further did not improve results 
significantly.
All the hidden layer weights were initialized randomly.

In Figure \ref{fig:expts} we show the normalized RMSE
(RMSE/[$F_{max}-F_{min}$]) for each of the neural networks for 
different values of the number of bodies $k$.

% \begin{table}[h!]
% \centering
% \begin{tabular}{ |p{1.6cm}||p{1.6cm}|p{1.6cm}|p{1.6cm}|  }
%  \hline
%  %\multicolumn{4}{|c|}{Normalized RMSE} \\
%  \hline
% $k$ (Number of Bodies) & ReLU & ReLU with bias & Gaussian Kernel\\
%  \hline
%  5 & 0.0036  & 0.041  & 0.029 \\
% 10 & 0.007  & 0.0065  & 0.034 \\
% 20 &  0.013  & 0.015  & 0.043\\
% 50 &  0.025  & 0.025  & 0.064\\
% 100 &  0.04  & 0.041  & 0.074\\
% 200 &  0.066  & 0.065  & 0.164\\
% 400 &  0.108  & 0.110  & 0.573\\
%  \hline
% \end{tabular}
% \caption{Normalized RMSE for learning gravitational
% force law for different kernels. $5\cdot10^{6}$ training examples, $\rrat = 10$.\label{table:expts}}
% \end{table}

\begin{figure}[h!]
\centering
\includegraphics[width=0.5\linewidth]{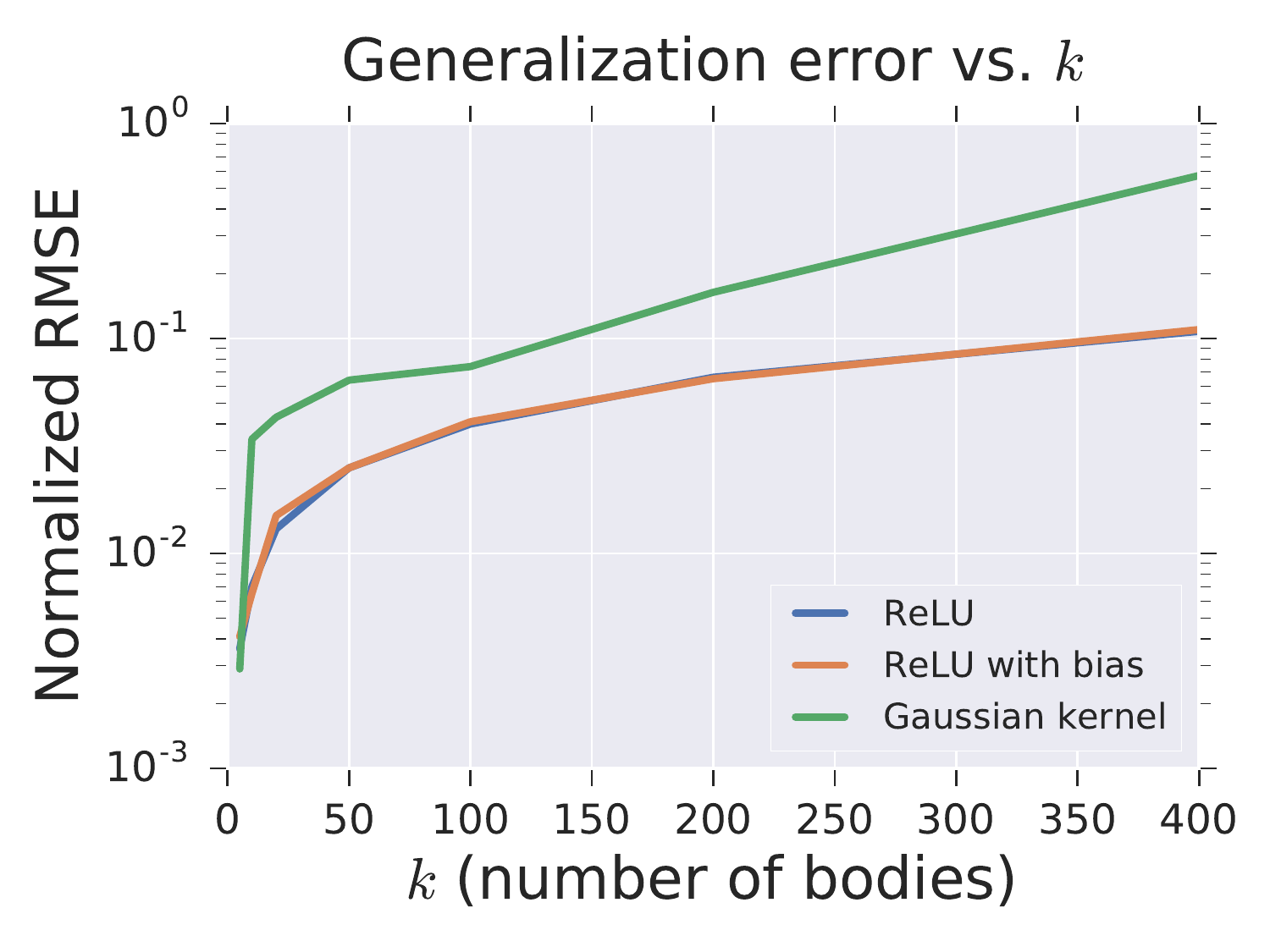}
\caption{RMSE vs number of bodies $k$ for learning gravitational
force law for different kernels. Normalized by the range $F_{max}-F_{min}$
of the forces. Gaussian kernels learn worse than
ReLU at large $k$.
}
\label{fig:expts}
\end{figure}

All three networks are able to learn the 
gravitational force equation with small normalized RMSE for hundreds of bodies.  
Both the ReLU network and ReLU with bias outperform the network corresponding to the Gaussian
kernel (in terms of RMSE) as $k$ increases.
In particular, the Gaussian kernel learning seems to quickly degrade at around $400$ bodies,
with a normalized RMSE exceeding $50\%$. This is consistent with the learning bounds
for these kernels in Section \ref{sec:analytic}, and suggests
that those bounds may in fact be useful to compare the performances
of different networks in practice.

We did not, however, observe much difference in the performance of the
ReLU network when adding a bias  to the input, which suggests that 
the inability to get an analytical bound due
to only even powers
in the ReLU NTK kernel might be a shortcoming of the proof
technique, rather than a property which fundamentally limits the
model.

\section{Lower Bounds}\label{sec:lower}

First, we show an exponential dependence on the depth $h$ is necessary for learning decision trees. The result depends on the hardness of solving parity with noise.

\begin{conjecture}
(hardness of parity with noise) Let $\ab, \xb\in\{0,1\}^d$ be $d$-dimensional Boolean vectors. In the parity with noise problem, we are given noisy inner products modulo 2 of the unknown vector $\xb$ with the examples $\ab_i$, i.e. $b_i=\ip{\ab_i}{\xb} +\eta_i \mod 2$ where $\eta_i$ is a Binomial random variable which is 1 with probability 0.1. Then any algorithm for finding $\xb$ needs at least $2^{\tilde{\Omega}( d)}$ time or examples (where $\tilde{\Omega}$ hides poly-logarithmic factors in $d$). Similarly, if $\xb$ is given to be $s$-sparse for $s\ll d$, then any  algorithm for finding $\xb$ needs at least $d^{{\Omega}( s)}$ time or examples.
\end{conjecture}

Note that the hardness of learning parity with noise is a standard assumption in computational learning theory and forms the basis of many cryptographic protocols \citep{regev2009lattices}. The best known algorithm for solving parity needs $2^{O(d/\log d)}$ time and examples \citep{blum2003noise}. Learning parities is also known to provably require $2^{\Omega(d)}$ samples for the class of algorithm known as \emph{statistical query algorithms}---these are algorithms are only allowed to obtain estimates of statistical properties of the examples but cannot see the examples themselves \citep{kearns1998efficient}. Note that the usual stochastic algorithms for training neural networks such as SGD can be implemented in the statistical query model \citep{song2017complexity}. Similar hardness result are conjectured for the problem of learning sparse parity with noise, and the best known algorithm runs in time $d^{\Omega(s)}$ \citep{valiant2015finding}.

Based on the hardness of parity with noise, we show that exponential dependence on the depth for learning decision trees is necessary.

\begin{theorem}\label{thm:parity}
Conditioned on the hardness of the sparse parity with noise problem, any algorithm for learning decision trees of depth $h$ needs at least $d^{\Omega( h)}$ time or examples. %any model which treats all coordinates symmetrically---and hence is invariant to a permutation applied to the coordinates of all the data points---needs at least $\Omega(d^b)$ time or samples to learn $f$. Hence a vanilla neural network which is invariant to a permutation of the coordinates needs at least $\Omega(d^b)$ time or samples to learn $f$. In contrast, learning a separate model for every set of the first $b$ bits needs $O(2^b)$ samples.
\end{theorem}
\begin{proof}
Note that we can represent a parity with noise problem where the answer is $h$-sparse by a decision tree of depth $h$ where the leaves represent the solutions to the parity problem. The result then follows by the hardness of the sparse parity with noise problem.
\end{proof}

We also show that the doubly exponential dependence on the depth for learning generalized decision programs is necessary.

\begin{theorem}\label{thm:depth_general}
Learning a generalized decision program which is a binary tree of depth $h$ using stochastic gradient descent requires at least $2^{2^{{\Omega}(h)}}$ examples. Conditioned on the hardness of learning noisy parities, any algorithm for learning a generalized program of depth $h$ needs at least $2^{2^{\tilde{\Omega}(h)}}$ time or examples (where $\tilde{\Omega}$ hides poly-logarithmic factors in $h$).
\end{theorem}

\begin{proof}
Note that a generalized decision program of depth $h$ can encode a  parity function over $D=2^h$ bits. Any statistical query algorithm to learn a parity over $D$ bits needs at least $2^{{\Omega}(D)}$  samples. As stochastic gradient descent can be implemented in the statistical query model, hence the bound for stochastic gradient descent follows.

To prove the general lower bound, note that a generalized decision program of depth $h$ can also encode a \emph{noisy} parity function over $D=2^h$ bits. Conditioned on the hardness of parity with noise, any algorithm for learning noisy parities needs at least $2^{\tilde{\Omega}(D)}$ samples. Hence the bound for general algorithms also follows.
\end{proof}

In our framework, we assume that all the underlying functions that we learn are analytic, or have an analytic approximation. It is natural to ask if such an assumption is necessary. Next, we show that learning even simple compositions of functions such as their sum is not possible without some assumptions on the individual functions.

\begin{lemma}
There exists function classes $F_1$ and $F_2$ which can be learnt efficiently but for every $f_1\in F_1$  there exists $ f_2\in F_2$ such that $f_1+f_2$ is hard to learn (conditioned on the hardness of learning parity with noise)
 \end{lemma}
\begin{proof}
Both $f_1$ and $f_2$ are modifications of the parity with noise problem. The input in both cases is $\x\in \{0,1\}^d$. Let $\bbet$ be the solution to the noisy parity problem. The output for the function class $F_1$ is $[\bbet, y]$, where $y$ is the value of the noisy parity for the input. The output for the function class $F_2$ is $[-\bbet, y]$, where $y$ is again the value of the noisy parity for the input. Note that $F_1$ and $F_2$ are trivial to learn, as the  solution $\bbet$ to the noisy parity problem is already a part of the output. For any $f_1\in F_1$, choose $f_2\in F_2$ to be the function with the same vector $\bbet$. Note that conditioned on the hardness of learning parity with noise, $f_1+f_2$ is hard to learn. 
\end{proof}

%$f_0(x)=frac(2x), x\in [0,1]$. $f_i=f_{i-1}\odot f_{i-1}$. Then $f_h(x)=frac(2^{2^h}x)$. In particular, to approximate $f_h$ by a degree $p$ polynomial we need $p\ge 2^{2^h}$. Specifically, if we want to output $step(f_h(<\alpha,x>))$ then this can be mapped to parity with noise as parity is equal to $ 2*frac((\sqrt{dk}/2)<\alpha,x>)$, where each $\alpha_i$ is either $1/\sqrt{k}$ or $0$ and each $x_i$ is either $1/\sqrt{d}$ or $0$. Setting $k$ so that $h=\log \log \sqrt{dk}$, we know that the complexity of learning $step(f_h)$ with noise is at least $d^k \ge d^{2^{2^{h}}/\sqrt{d}}$, for any $\log \log \sqrt{d}\le h\le \log \log d$.

\subsection{Lower bounds for learning any analytic function}

In this section, we show that there is a lower bound on the Rademacher complexity $\bar{\yb}^T \bar{H}^{-1} \yb$ based on the coefficients in the polynomial expansion of the $\tilde{g}$ function. Hence the $\tilde{g}$ function characterizes the complexity of learning $g$.

For any $J=(J_1,\dots, J_n) \in \mathbb{N}^n$, write a monomial $X_J=x_1^{J_1}\dots x_n^{J_n}$. Define $|J|=\sum_k J_k$. For a polynomial $p(x)= \sum_J a_J x_J$, where $a_J\in \mathbb{C}$, its degree $\deg(p) = \max_{a_J\ne 0} |J|$. The following fact shows that monomials form an orthogonal basis over the unit circle in the complex plane.

\begin{fact}\label{fact:complex_orth}
$\ip{X_{J}}{X_{J'}}_{\mathbb{C}^n} = 1$ if $J=J'$ and 0 otherwise (here, $\ip{\cdot}{\cdot}_{\mathbb{C}^n}$ denotes the inner product over the unit circle in the complex plane).
\end{fact}

Note that according to Theorem \ref{thm:univar} the sample complexity for learning $g(x)$ depends on $\tilde{g}'(1) = \sum_j j|a_j|$, and hence is the $\ell_1$ norm of the derivative. The following Lemma shows that this is tight in the sense that $\Omega(\sum_j j a_j^2)$ samples or the $\ell_2$ norm of the derivative are necessary for learning $g(x)$.

For any variable $x$ let $\bar{x}$ denote the complex conjugate of $x$. Let $\xb_1, \xb_2, \dots, \xb_n$ denote the training examples. Let $Q$ denote the kernel polynomial so that $K(\xb_i, \xb_j)= Q(\bar{\xb_i}^T \xb_j)$. Let $Q(t) = \sum_i q_i t^i$. For simplicity, let us look at the case where the power series and the kernel polynomial are univariate polynomials of a bounded degree $\deg(q)$. We will assume that we have enough samples that Fact \ref{fact:complex_orth} hold when averaging over all samples. Let $q_J$ be the coefficient of $T_J$ in the polynomial expansion of $Q(t_1+ \dots + t_n)$.

\begin{lemma}
For a univariate polynomial $y=p(x)$ , $ \bar{\yb}^T {H}^{-1} \yb = \sum_j a_j^2/q_j$ asymptotically in the sample size, where $a_j$ are the coefficients of the polynomial $p$. For a multivariate polynomial, $ \bar{\yb}^T {H}^{-1} \yb = \sum_j a_J^2/q_J$ asymptotically in the sample size. Here, ${H}^{-1}$ denotes the pseudoinverse of ${H}$.
\end{lemma}
\begin{proof}
 
We will begin with the univariate case. Let $\{(x_1,y_1), (x_2,y_2, \dots, (x_n,y_n)\}$ denote the training examples and their labels. Let $\yb$ be the vector of all the labels $\{y_i\}$. Let $d=\max\{\deg(p),\deg(q)\}$ (where we assume that $\deg(q)$ is bounded for simplicity). Now consider the matrix $G$ with $n$ rows and $d$ columns where the $(i,j)$-th entry is $x_i^j$. Note that $\bar{G}^T$ transforms $\yb$ from the standard basis to the monomial basis, i.e. the expected value of $(1/n) \bar{G}^T \yb$ is $(a_1, \dots, a_d)$ (by Fact \ref{fact:complex_orth}). Therefore,  $(1/n) \bar{G}^T \yb=(a_1, \dots, a_d)$ asymptotically in the sample size $n$. We claim that $H= {G} D \bar{G}^T$ where $D$ is the diagonal matrix where $D_{k,k}= q_k$. To verify this, let $G_{(i)}$ denote that $i$-th row of $G$ and observe that the $(i,j)$-th entry ${G}_{(i)} D \bar{G}_{(j)}^T= \sum_k {x}_i^k q_k \bar{x_j}^k=q_k({x_i} \bar{x}_j)^k=K(x_i, x_j)=H_{i,j}$. Now given the orthonormality of the monomial basis, $(1/n)\bar{G}^T{G} =I$. Therefore since $H= {G} D \bar{G}^T$ is the SVD of $H$, ${H}^{-1}= (1/n^2){G} D^{-1} \bar{G}^T$. Hence $ \bar{\yb}^T {H}^{-1} \yb = {((1/n){G^T\bar{\yb}})}^T D^{-1} ((1/n)\bar{G}^T{\yb})= \sum_j (1/q_j) a_j^2 $. 
%\bar{((1/n)\bar{Gy})}^T D^{-1} ((1/n)\bar{Gy})

%as ${H}^{-1} {H} y=G \bar{G}^T y=y$ (as $\bar{G}^T$ will transform $y$ from the standard basis to the monomial basis, and applying $G$ on the result will transform it back to the standard basis)

For the multivariate case, instead of having $d$ columns for $G$, we will have one column for every possible value of $J$ of degree at most $d$. In the diagonal entry $D_{J,J}$ we put $q_J$, where $q_J$ is the coefficient of $T_J$ in the polynomial expansion of $Q(t_1+ \dots + t_n)$.
\end{proof}

\begin{corollary}
\label{cor:lower}
For the ReLU activation $q_j=\Omega(1/j)$, and hence $ \bar{\yb}^T \bar{H}^{-1} \yb \ge \Omega(\sum_j j a_j^2)$ asymptotically in the sample size.
\end{corollary}

Note that in Theorem \ref{thm:univar}, the upper bound for the sample complexity was  $O(\sum_j j |a_j|)$, hence Theorem~\ref{thm:univar} is tight up to the distinction between the $\ell_1$ and $\ell_2$ norm (which can differ by at most $\sqrt{\deg(p)}$). 

\section{Additional Details for Experiments}\label{sec:appendix_app}

\subsection{Setup details}

All the experiments are done in TensorFlow, trained with a GPU accelerator. We use the default TensorFlow values for all hyper parameters involved in the training of the neural networks. All the experiment results averaged over $3$ runs. The number of training epochs for each experiment and average runtime (for one run) are summarized in Table \ref{tab:exp_runtime}. For cluster experiments, number of training examples per cluster varies $1000$ to $100000$, average runtime varies from $2$ minutes to $100$ minutes.  For the decision tree experiments, number of training examples per leaf node varies from $64$ to $512$, avarage runtime varies from $14$ minutes to $42$ minutes. For the SQL-style aggregation experiment, the train dataset contains $16384$ examples, and test dataset contains $4096$ examples, average runtime is $50$ minutes. The source for the Penn World Table dataset \cite{feenstra2015next} used in the SQL query experiment is  \url{https://www.rug.nl/ggdc/productivity/pwt/} and it is also available at \url{https://www.kaggle.com/jboysen/penn-world-table}. 

\begin{table}[h]
    \centering
    \caption{Number of epochs and average runtime}
    \label{tab:exp_runtime}
    % neurips style file is asking that table headings be above table
    \begin{tabular}{@{}lcc@{}}
    \toprule
    Experiment name & Number of epochs & Average runtime\\
    \midrule
    Cluster & $100$ & $2$ - $100$ minutes \\
    \midrule
    Decision Tree & $200$ & $14$ - $42$ minutes  \\
    \midrule
    SQL-style aggregation & $6400$ & $50$ minutes \\
    \bottomrule
    \end{tabular}
    \vspace{2pt}
\end{table}

\subsection{Additional details for learning clusters of linear functions}\label{sec:lin}

We provide a more detailed setup of the experiment reported in Fig. \ref{fig:nn1} where the task codes are given by clusters, and there is a separate linear function for every cluster. In this experiment, the data is drawn from $k$ clusters, and from a mixture of two well-separated Gaussians in each cluster. Data points from the two Gaussians within each cluster are assigned two different labels, for $2k$ labels in total. Fig. \ref{fig:lin1} below shows an instance of this task in two dimensions, the red circles represent the clusters, and there are two classes drawn from well-separated Gaussians from each cluster. In high dimensions, the clusters are very well-separated, and doing a $k$-means clustering to identify the $k$ cluster centers and then learning a simple linear classifier within each cluster gets near perfect classification accuracy.  Fig. \ref{fig:lin2} shows the performance of a single neural network trained on this task (same as Fig. \ref{fig:nn1} in the main body). We can see that a single neural network still gets good performance with a modest increase in the required number of samples.

\begin{figure}
\centering
\begin{subfigure}{0.4\textwidth}
    \centering
    \includegraphics[width=\textwidth]{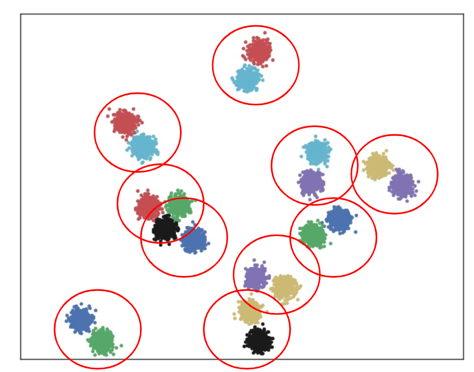}
    \caption{An instance of the problem with multiple clusters, each cluster is indicated by a red circle.}
    \label{fig:lin1}
\end{subfigure}
\begin{subfigure}{0.4\textwidth}
    \centering
    \includegraphics[width=\textwidth]{cluster_nn.pdf}
    \caption{Test accuracy vs. number of points per cluster}
    \label{fig:lin2}
\end{subfigure}
\label{fig:lin}
\caption{Experiment where data is clustered into tasks with a separate linear function for each task. A single neural network does well even when there are multiple clusters.}
\end{figure}

\end{document}